\documentclass[lettersize,journal]{IEEEtran}
\usepackage{amsmath,amsfonts}
\usepackage{algorithmic}
\usepackage{algorithm}
\usepackage{array}
\usepackage[caption=false,font=normalsize,labelfont=sf,textfont=sf]{subfig}
\usepackage{textcomp}
\usepackage{stfloats}
\usepackage{url}
\usepackage{verbatim}
\usepackage{graphicx}
\usepackage{cite}

\usepackage{amssymb}
\usepackage{amsthm}
\usepackage[mathscr]{eucal}
\usepackage{mathtools}
\usepackage{dsfont}
\usepackage{makecell}

\theoremstyle{plain}
\newtheorem{theorem}{Theorem}[section]
\newtheorem{proposition}[theorem]{Proposition}
\newtheorem{lemma}[theorem]{Lemma}
\newtheorem{corollary}[theorem]{Corollary}
\theoremstyle{definition}

\newtheorem{assumption}[theorem]{Assumption}
\theoremstyle{remark}
\newtheorem{remark}[theorem]{Remark}

\begin{document}

\title{Graph Structured Combinatorial Semi-Bandit with Nonlinear Reward Associations through Separable Signals}

\author{%
  \IEEEauthorblockN{Christoph Bauschmann and Setareh Maghsudi} \\
 \IEEEauthorblockA{Department of Electrical Engineering and Information Technology \\
                    Ruhr University Bochum\\
                    Bochum, Germany}
                    \thanks{This work has been submitted to the IEEE for possible publication. Copyright may be transferred without notice, after which this version may no longer be accessible.}}

\maketitle

\begin{abstract}
	The identification of optimal structures within vast arrays of interconnected data necessitates significant sampling- and computational effort. Learning and leveraging underlying signal dependencies can improve efficiency and predictive capabilities considerably, but the ubiquity of nonlinear statistical relations amplifies the complexity of such undertakings. In this paper, we develop novel generic and adaptive strategies equipped with routines for graph-based causal reward modeling, analytic reproducing kernel methods, and Taylor approximation of functional processes. We establish theoretical performance guarantees sublinear in time and linear in data volume over time. Our analyses cover robustness to a multitude of uncertainties arising from noise interference, gradual model convergence, and solution space mismatch. The framework's general appeal is substantiated by a minimalistic set of conditions or reliance on prior estimates, while various outlined modifications address specific or extended settings. To demonstrate practical effectiveness, we conduct numerical experiments using both benchmarked synthetic and real-world transportation datasets.
\end{abstract}

\section{Introduction}
\IEEEPARstart{S}{equential} decision-making processes without influence on the state transition constitute the class of \textit{Multi-Armed Bandit} (MAB) problems. Upon selection, each arm provides a reward, which may follow a \textit{stochastic} distribution. The agent's goal is to maximize the reward accumulated over the game horizon by adapting her actions based on the interaction history. The performance of the controlling policy hinges on its ability to balance queries to underexplored distributions in hopes of future payoff against exploitation of sampled answers for immediate gain. Formulated as a minimization objective, the cumulative difference in scoring to an optimal policy is understood to be the agent's regret.

When the agent is allowed to pull multiple (base) arms simultaneously, the viable subsets become the super arms of a \textit{combinatorial} MAB. The posterior reward distributions of intersecting super arms are associated through their joint dependence on the underlying base arm distributions. If the correlation is linear, the reward is simply a weighted sum of the base signals. In general, causal relationships can take any nonlinear form without ad hoc decomposition. However, an achieving agent will endeavor to make predictive links between similar sets of arms. The potential advantages over conventional MAB techniques that keep individual records are especially pronounced for large numbers of base arms as their power set grows exponentially.
 
In a \textit{semi-bandit} setting, the agent's feedback can entail auxiliary data besides the actual reward. In the case of a combinatorial MAB, natural candidates are the samples from the chosen base arms' distributions that determine the ultimate reward. If the reward calculation has a certain \textit{structure}, intermediate values of that computation might become available. In fact, without restrictions to the set of possible functions, a frequentist sub-binomial regret expectation cannot be guaranteed, as each valid base arm combination brings unique coefficients into effect. In a sense, off-policy evaluation shall be supported in the solution space.

One such framework to formalize causal interdependencies are \textit{Structural Equation Models} (SEMs). They work based on a graph representation that aggregates initial node values over weighted edges. In the bandit context, the agent filters the base signal values as initial impulses to the SEM, and the resulting output is accumulated in the overall reward. First introduced in \textup{\cite{ijcai2022p676}} exclusively for linear edge weights, we expand this notion to general activation functions between vertices and by a \textit{kernelized} optimization that does not rely on perfect feedback for exact weight identification.

\textbf{Literature Overview.} Algorithms in the relevant literature need to be distinguished by the assumptions they operate under. There are many ways in which the \textit{uncertainty} of the reward generation process may be limited. Often, they take the form of prior expertise or access to an oracle to approximate the results, as done for the general analysis carried out in \textup{\cite{NIPS2016_aa169b49}}. Moreover, the randomness may be abstracted by concentrating on special signal distributions or feedback that deviates little from its expected value. For instance, \textup{\cite{pmlr-v70-chowdhury17a}} have the differences between mean and observation follow a subgaussian distribution homogeneous across super arms and deploy Gaussian processes and likelihood models. The neural-network-based bandits in \textup{\cite{pmlr-v202-hwang23a}} rely on contextual information that captures all correlations between arms. Additionally, many works do not explicitly account for noise interference. While unstructured approaches may subsume noise independent between arms in the stochastic reward function, structure identification could exploit unrealistic perfect observations.

The conditions set for the \textit{causality} of the process also warrant attention. The control the agent can exert on the system can range from masking signals to soft/hard intervention on their values \textup{\cite{pmlr-v238-yan24a}}, \textup{\cite{Feng_Chen_2023}}. Causal graphs as in \textup{\cite{NEURIPS2021_d010396c}}, \textup{\cite{pmlr-v275-konobeev25a}} additionally introduce independence prerequisites such as faithfulness or effect identifiability. Their underlying logic is commonly framed with linear SEMs \textup{\cite{JMLR:v24:22-0969, NEURIPS2024_2aba6ec2}}. Besides linearity, special properties like monotonicity \textup{\cite{NIPS2016_aa169b49}} or submodularity \textup{\cite{osti_10579433}} may be imposed upon the functions describing the reward associations. In \textup{\cite{pmlr-v206-aouali23a}}, intermediate rewards are sampled along a hierarchy of parametrized distributions rooted in latent priors. The authors develop a Thompson Sampling (TS) approach and investigate specific examples of nonlinearities only for the final accumulating step. The covariance estimations in \textup{\cite{NEURIPS2024_3640a199}} capture arbitrary associations that precede the selection, but they do not encompass nonlinearities after the action filter. A different graphical framework is that of \textit{probabilistically triggered neighboring arms} \textup{\cite{pmlr-v202-kocak23a}}, where high connectivity is beneficial rather than a hindrance to the complexity of the learning process. Extensive foundations and an overview of problem classes are supplied by reference books such as \textup{\cite{LaSze19:book}}.

\textbf{Contribution.} For many of these individual scopes, regret bounds that are sublinear in time steps and at most linear in the number of available base arms (for time-independent terms) are established in the respective references (for a tabular comparison, see Appendix \ref{sec:bac}). However, we contribute the first policies for the stochastic structured combinatorial semi-bandit setting to achieve such bounds when the pairwise associations between signals follow arbitrary analytic functions. To simultaneously discover and infer from this flexible description, the routines we develop negotiate the duplicate exploration of inherent distributions and their causal dynamics with the objective. We refrain from prohibitive restrictions to uncertainty and investigate the effect noise, structural misalignments and functional complexity have on convergence. For greater applicability, we delineate alterations that address related or special scenarios. Our theoretical findings and algorithmic analysis are corroborated by numerical tests on benchmarked synthetic and real-world data.

Combinatorial MABs accommodate a wide variety of tasks, such as drug discovery \textup{\cite{pmlr-v85-durand18a}}, online recommender systems \textup{\cite{pmlr-v206-aouali23a}}, neural network design \textup{\cite{NNsearch}} and analysis (e.g., identifying dominant subnetworks or sets of biases), social influence maximization \textup{\cite{osti_10579433}} and live routing (e.g., preserving resources for vehicles or wireless communication). The SEM can be most readily applied when inert and aggregated forms of data are apparent or extractable (e.g., from phases of separation or abnormal behavior within the network). It's capable of describing spatial graphs and signal processing thereon (e.g., of logistical/biological nature like transportation delays respective infection spreads). A motivating example lies in the distribution of variable grid load. Depending on the current local consumption and their proximity, producers can complement or interfere with each other's efficiency in intricate ways. An agent would have to decide which contractors to enroll so that a sustainable coverage is ensured at the lowest possible cost.

\textbf{Outline.} Over the course of this paper, we first provide the formal specifications in Section \ref{sec:profor}, before developing the suggested decision-making policies in Section \ref{sec:decstr}. In Section \ref{sec:theana}, we then prove regret bounds and account for sources of inaccuracy. Subsequently, we analyze the experimental performance in Section \ref{sec:expana}. Finally, we draw conclusions and discuss future research directions in Section \ref{sec:con}.

\textbf{Notation.} We use bold lower/upper case letters for vectors and matrices, respectively. We apply the common multi-index notation $\mathbf{a}^{\mathbf{b}}=\prod_{i\in[N]} \mathbf{a}[i]^{\mathbf{b}[i]}$, $\mathbf{a}!=\prod_{i\in[N]} \mathbf{a}[i]!$, $\vert\mathbf{a}\rvert=\sum_{i\in[N]}\lvert\mathbf{a}[i]\rvert$, and $\mathbf{a} \text{ mod } k = (\mathbf{a}[i] \text{ mod } k)_{i\in[N]}^{\top}$ for $\mathbf{a},\mathbf{b}\in\mathbb{N}^N$, $k,N \in\mathbb{N}$. We write $\mathbf{a}\odot\mathbf{b} = \text{diag}(\mathbf{a})\mathbf{b}$ for the Hadamard product of vectors.
%
\section{Problem Statement}
\label{sec:profor}
\noindent We represent the set of base arms with $[N]=\{1,..,N\}$. Let $\mathcal{B}_1, \dots, \mathcal{B}_N$ be (a priori unknown) probability distributions on closed intervals of $\mathbb{R}$. At each time step $t$, the vector of independent \textit{instantaneous rewards} $\mathbf{b}_t$ is drawn from the composite distribution $\mathcal{B}$ on $[0,1]^N$. To stay concise, we restrict our description to absolutely continuous distributions on $[0,1]$ that have nonzero lower bounds to their probability density function (p.d.f.). Shifts to other bounded ranges can w.l.o.g. be incorporated into the objective function. Unrestricted distributions and p.d.f.s are discussed at the end of Section \ref{sec:theana} respective in Appendix \ref{sec:ext} (Remark \ref{rem:disbou}).

For the combinatorial bandit problem, an agent is tasked with sequentially selecting subsets of base arms over time. Each such \textit{super arm} is characterized by the agent's \textit{decision vector} $\mathbf{x}_t\in\{0,1\}^N$ over the inclusion of each base arm, with the base arm $i$ being chosen iff $\mathbf{x}_t[i]=1$. The total cardinality of valid super arms is limited by $s\in\mathbb{N}$, thus restricting the set of allowed super arms to \[\mathcal{X}=\{\mathbf{x}\in\{0,1\}^N~|~\|\mathbf{x}\|_0 \leq s\},\] with $\|\cdot\|_0$ counting the nonzero elements of a vector.

We use a (sparse) directed non-self-cyclic graph $\mathcal{G}=([N],\mathcal{E},\mathbf{F})$ on the base arms with edge set $\mathcal{E}$ to model the process by which the reward associated with each super arm is calculated. We collectively represent the functional weights for each edge in a matrix-styled operator $\mathbf{F}$ of dimension $(N,N)$ that maps $\mathbf{b}\in\mathbb{R}^N$ to $\mathbf{F}(\mathbf{b}) = (\sum_{j=1}^{N}\mathbf{F}[i,j](\mathbf{b}[j]))_{i\in[N]}^{\top}$. All component functions $\mathbf{F}[i,j]: \mathbb{R}\rightarrow\mathbb{R}$ are analytic. In this sense, we also understand $\mathbf{I}=\text{diag}(\text{id})$ as the identity mapping, and conventionalize concatenation as the default connection between operators.

More specifically, we assume the causal relationships follow an additively separable SEM \textup{\cite{topology}}. The \textit{exogenous input vector} to this model is
\begin{align}
\label{eq:exo}
	\mathbf{z}_t=\text{diag}(\mathbf{b}_t)\mathbf{x}_t = \mathbf{x}_t\odot\mathbf{b}_t.
\end{align}
Its \textit{endogenous output vector} is subsequently acquired as 
\begin{align}
\label{eq:endo}
	\mathbf{y}_t[i] = \sum\limits_{j\neq i}\mathbf{F}[i,j](\mathbf{y}_t[j])+\mathbf{G}[i,i](\mathbf{z}_t[i]) ,\forall i\in[N].
\end{align}
Here, $\mathbf{G}$ is a diagonal matrix-styled operator of exogenous activation functions, while $\mathbf{F}$ acts endogenously. The solution $\mathbf{y}_t$ to the above set of equations constitutes the vector of \textit{overall rewards} in our bandit setup. Hence, each of its entries is conditioned on the associated instantaneous rewards and the overall reward of its immediate neighbors in $\mathcal{G}$. Since $\text{diag}(\mathbf{F})={\bf 0}$, the vectorized form of (\ref{eq:endo}) reads as 
\[\mathbf{y}_t=\mathbf{F}(\mathbf{y}_t)+\mathbf{G}(\mathbf{z}_t)=\mathbf{F}(\mathbf{y}_t)+\mathbf{G}(\mathbf{x}_t\odot\mathbf{b}_t).\] 
When the operator $\mathbf{I}-\mathbf{F}$ is invertible, this equation can be formulated explicitly as 
\[\mathbf{y}_t = (\mathbf{I}-\mathbf{F})^{-1}(\mathbf{G}(\mathbf{x}_t\odot\mathbf{b}_t)).\] 
Unlike the instantaneous rewards, the overall rewards may take on values in vastly different ranges per dimension.
\begin{figure}[!t]
    \centering
    \includegraphics*[width=0.95\columnwidth]{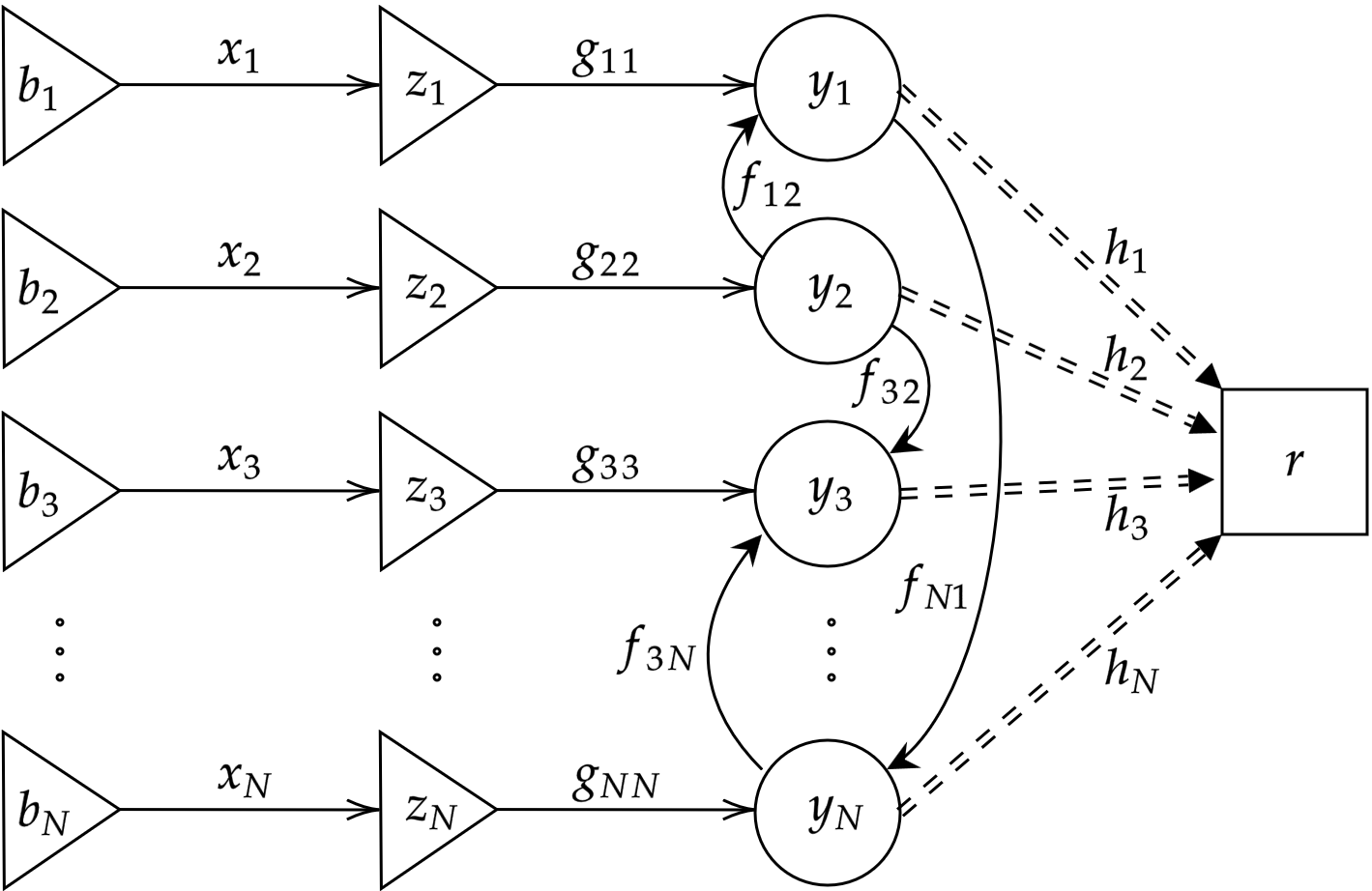}
    \caption{Causal Reward Relations Following a Separable SEM.}
    \label{fig:SEM_relations}
\end{figure}

After selecting a super arm, the agent receives the strong semi-bandit feedback of the SEM's exogenous and endogenous vectors for that time step. Following the bandit framework, the extracted \textit{payoff}
\[\mathbf{r}(\mathbf{x}_t,\mathbf{b}_t) = \mathbf{r}_{\mathbf{x}_t}(\mathbf{b}_t) = \mathbf{h}^{\top}\mathbf{y}_t = \mathbf{h}^{\top}(\mathbf{I}-\mathbf{F})^{-1}(\mathbf{G}(\mathbf{z}_t)),\] 
for an operator $\mathbf{h}$ of $N$ functions, serves as the agent's objective function to be maximized. Equivalently, the agent aims to minimize the \textit{expected regret} \[\mathcal{R}(T) = T\mu(\mathbf{x}^{\ast})-\sum\limits_{t=1}^T\mu(\mathbf{x}_t)\] when comparing the \textit{expected payoff} \[\mu(\mathbf{x})=\mathbb{E}_{\mathbf{b}\sim\mathcal{B}}[\mathbf{r}_{\mathbf{x}}(\mathbf{b})]\] at each time step up to unknown horizon $T\in\mathbb{N}$ with that of a sought after \textit{optimal decision vector} $\mathbf{x}^{\ast}\in\underset{\mathbf{x}\in\mathcal{X}}{\arg\max}~\mu(\mathbf{x})$.

The reward structure depicted in Figure \ref{fig:SEM_relations} displays all potential functional components. For ease of exposition, the main paper's focus lies on $\mathbf{G}=\mathbf{I}$ for the preliminary on-node activations, a linear final accumulation $\mathbf{h}^{\top}=(\text{id},\dots,\text{id})$, which we shall denote simply by $\mathbf{r}_{\mathbf{x}_t}(\mathbf{b}_t)={\bf 1}^{\top}\mathbf{y}_t$, and zero-true $\mathbf{F}$, i.e. $\mathbf{F}({\bf 0})={\bf 0}$. We generalize these settings and describe further extensions to the framework, including partial endogenous feedback, in Appendix \ref{sec:ext}.
%
\section{Decision-Making Strategy}
\label{sec:decstr}
\noindent In this section, we introduce our policies for the problem described before. Since we do not assume knowledge of the causal dynamics, the agent has to trade off two sources of uncertainty against exploiting its insights. The first relates to the graph functions in $\mathbf{F}$, while the second regards the signal distribution $\mathcal{B}$. Our proposed strategy consists of three distinct parts, which we describe below.
\subsection{Online Graph Learning}
\noindent Our algorithms maintain an explicit world model through the approximation of $\mathbf{F}$ in a pre-selected space of solutions. For $i\in[N]$, let $\mathcal{H}^{(i)}$ be a reproducing kernel Hilbert space (RKHS) with continuous positive-definite kernel, consisting of analytic, zero-true functions from a closed interval in $\mathbb{R}^N$ to $\mathbb{R}$ that are additively separable over their $N$ arguments and zero-imaged on $\mathbb{R}\mathbf{I}[i]^{\top}$. We use $\mathcal{H}$ to denote their composition to a space of operators on $\mathbb{R}^N$. For $\theta>0$, we set $\mathcal{H}^{(i)}_{\theta} = \{\mathbf{u}\in\mathcal{H}^{(i)}~|~\|\mathbf{u}\|_{\mathcal{H}^{(i)}}^2\leq\theta\}$. By $\mathcal{H}_{\theta}$, we summarize the space of operators restricted in each component of $\mathcal{H}$.

We collect the feedback up to time $t$ using the matrices $\mathbf{Z}_t=[\mathbf{z}_1,\dots,\mathbf{z}_t]$ and $\mathbf{Y}_t=[\mathbf{y}_1,\dots,\mathbf{y}_t]=\mathbf{F}(\mathbf{Y}_t)+\mathbf{Z}_t$. 
Through kernel optimization, we obtain
\begin{align}\label{eq:joint}
    \hat{\mathbf{F}}_{t} &=\underset{\mathbf{H}[i]\in\mathcal{H}^{(i)}}{\arg\min}~\frac{1}{t}\|\mathbf{H}(\mathbf{Y}_{t})-\mathbf{Y}_{t}+\mathbf{Z}_{t}\|_F^2 \\\nonumber
    &\hspace{5mm}+\lambda\sum\limits_{i\in[N]}\|\mathbf{H}[i]\|_{\mathcal{H}^{(i)}}^2,
\end{align}
where $\|\cdot\|_F$ and $\|\cdot\|_{\mathcal{H}^{(i)}}$, $i\in[N]$, are the Frobenius and Hilbert space norms, respectively. The factor $\lambda>0$ regularizes the sparsity of the resulting solution. By Lagrangian duality, it can be represented through a convex restriction
\begin{align}\label{eq:conv_rest}
    \|\mathbf{H}[i]\|_{\mathcal{H}^{(i)}}^2 \leq \theta~\forall i\in[N].
\end{align}
The optimization is independent between dimensions $i\in[N]$ and can also be performed in the respective subspaces. In practice, sparsity or bijectivity of $\mathbf{I}-\hat{\mathbf{F}}_t$ can be maintained through edge recognition thresholds.

Rather than performing a joint optimization over the entire sample set, we also consider computing a separate $\mathbf{F}_t^{(o)}$ for each $o\in[N]$ based solely on the time steps $\tau$ in which $\mathbf{x}_{\tau}[o]=1$. Algorithms with this variant then assemble the entries of their overall model recursively for $k\in[o-1]$ as
\begin{align}\nonumber
	&\hat{\mathbf{F}}_t[o-k,o]: \mathbb{R}\rightarrow\mathbb{R}, b \mapsto (\mathbf{I}-\hat{\mathbf{F}}^{(o)}_t)^{-1}(b\mathbf{I}[o]^{\top})[o-k]\\\label{eq:disjoint}
	&\hspace{3mm}-\sum\limits_{j=o-k+1}^{o-1}\hat{\mathbf{F}}_t[o-k,j]((\mathbf{I}-\hat{\mathbf{F}}^{(o)}_t)^{-1}(b\mathbf{I}[o]^{\top})[j]).
   \end{align}
Should the indexing in (\ref{eq:disjoint}) not be inherent to the application, edge recognition thresholds should be increased until a DAG is formed. We will elaborate on the chosen method and rationale behind it when presenting the individual algorithms in \ref{sec:Algo}.
\subsection{Internal UCB}
\label{subsec:int}
\noindent To handle the exploration-exploitation trade-off for signal passes through the learned model, the algorithms keep track of empirical moments
\begin{align}\label{eq:moments}
    \hat{\boldsymbol{\phi}}^{(p)}_{t}[i] = \sum_{\tau = 1}^{t} (\mathbf{z}_{\tau}[i]-\mathbf{a}_{t+1}[i])^p \mathds{1}\left\{ \mathbf{x}_{\tau}[i] = 1\right\} / \mathbf{m}_{t}[i]
\end{align}
about some $\mathbf{a}_{t+1}\in[0,1]^N$ up to a preselected degree $[q]\ni p$, of the observed instantaneous rewards.

They keep counters $\mathbf{m}_t=\sum\limits_{\tau=1}^t \mathbf{x}_{\tau}$ for the number of rounds each base arm was part of the selected super arms up to time $t$. After including each base arm at least once, the internal confidence interval has a radius of
\begin{align}
\label{eq:confidence}
    \mathbf{C}_t[i]=\sqrt{(s+1)\ln(t) / \mathbf{m}_t[i]} ,\forall i\in[N].
\end{align}
To leverage this confidence, the expected payoff needs to be expressed in terms of the distribution invariants (\ref{eq:moments}). Our approach is based on a representation of the reward function with its Taylor series expansion around any center point $\mathbf{a}\in[0,1]^N$ for $\mathbf{b}\in[0,1]^N$, $\mathbf{x}\in\mathcal{X}$:
\begin{align}\nonumber
	\mathbf{r}_{\mathbf{x}}(\mathbf{b}) &= \sum\limits_{\substack{\lvert{\boldsymbol{\alpha}}\rvert = 0\\\mathbf{x}\odot{\boldsymbol{\alpha}}={\boldsymbol{\alpha}}}}^{\infty}\frac{D^{{\boldsymbol{\alpha}}}\mathbf{r}_{\bf 1}(\mathbf{a}\odot\mathbf{x})}{{\boldsymbol{\alpha}}!}(\mathbf{b}-\mathbf{a})^{{\boldsymbol{\alpha}}}.
\end{align}
Several of the initial derivatives $D^{{\boldsymbol{\alpha}}}\mathbf{r}_{\mathbf{x}} = \frac{\partial^{\lvert{\boldsymbol{\alpha}}\rvert}}{\partial\mathbf{b}[1]^{{\boldsymbol{\alpha}}[1]}\dots\partial\mathbf{b}[N]^{{\boldsymbol{\alpha}}[N]}}\mathbf{r}_{\mathbf{x}}$ vanish whenever $\mathbf{x}\odot{\boldsymbol{\alpha}}\neq{\boldsymbol{\alpha}}$ for ${\boldsymbol{\alpha}}\in\mathbb{N}^N$. The independence of the underlying distributions also allows us to establish
\begin{align}\nonumber
	\mu(\mathbf{x}) &= \sum\limits_{\substack{\lvert{\boldsymbol{\alpha}}\rvert = 0\\\mathbf{x}\odot{\boldsymbol{\alpha}}={\boldsymbol{\alpha}}}}^{q}\frac{D^{{\boldsymbol{\alpha}}}\mathbf{r}_{\mathbf{x}}(\mathbf{a})}{{\boldsymbol{\alpha}}!}\prod\limits_{i\in[N]}\boldsymbol{\phi}^{({\boldsymbol{\alpha}}[i])}[i] + \mathbb{E}_{\mathcal{B}}[R_{q+1}(\mathbf{b})],
\end{align}
for the series approximation of degree $q$ and moments $\boldsymbol{\phi}^{(p)}[i]=\mathbb{E}_{\mathcal{B}_i}[(b-\mathbf{a}[i])^p]$, $i\in[N]$, $p\in[q]$. The remainder $R_{q+1}$ is bounded in absolute value by $\sum\limits_{\lvert{\boldsymbol{\alpha}}\rvert=q+1}\frac{1}{{\boldsymbol{\alpha}}!}\underset{\lvert{\boldsymbol{\gamma}}\rvert=\lvert{\boldsymbol{\alpha}}\rvert}{\max}\underset{\mathbf{c}\in[0,1]^N}{\max}\lvert D^{{\boldsymbol{\gamma}}}\mathbf{r}_{\bf 1}(\mathbf{c})\rvert$ (\textup{\cite{analysis}}).

To boost internal exploration of the current reward model $\hat{\mathbf{r}}^{(t-1)}=(\mathbf{I}-\hat{\mathbf{F}}_{t-1})^{-1}$, the algorithms optimistically project values within the moments' confidence interval that maximize the estimated payoff for each arm $\mathbf{x}\in\mathcal{X}$:
\begin{align}\nonumber
	&I_{t}^{\max}(\mathbf{x})=\underset{\mathcal{W}^{(p)}_{t-1}\in\mathbb{R}^N: ~\left\lvert{\mathcal{W}}^{(p)}_{t-1}[i]\right\rvert=2^{p\text{ mod }2}\mathbf{C}_{t-1}[i]~\forall p\in[q],i\in[N]}{\max} \\ \nonumber
        &\hspace{0mm}\sum_{\substack{\lvert{\boldsymbol{\alpha}}\rvert\leq q\\\mathbf{x}\odot\boldsymbol{\alpha}=\boldsymbol{\alpha}}} \frac{D^{{\boldsymbol{\alpha}}}\hat{\mathbf{r}}^{(t-1)}_{\mathbf{x}}(\mathbf{a}_t)}{{\boldsymbol{\alpha}}!} \prod_{j:\mathbf{x}[j]=1}\left(\hat{\boldsymbol{\phi}}^{({\boldsymbol{\alpha}}[j])}_{t-1}[j] + \mathcal{W}^{({\boldsymbol{\alpha}}[j])}_{t-1}[j]\right).
\end{align}
The specific growth behavior of the functions from the RKHS can be used to narrow the range of potentially maximizing values within the confidence interval. For example, if the functions have only non-negative derivatives and all moments are non-negative, one can substitute $\mathcal{W}_t^{(p)}=2^{p\text{ mod }2}\mathbf{C}_t ,\forall p\in[q]$. Under these circumstances, only super arms of cardinality $s$ would be candidates for maximizing the expected payoff.

To center the Taylor approximation around relevant values, $\mathbf{a}_{t+1}$ can be chosen as the empirical means
\begin{align}
\label{eq:means}
    \hat{\boldsymbol{\beta}}_{t}[i] = \sum_{\tau = 1}^{t} \mathbf{z}_{\tau}[i] / \mathbf{m}_{t}[i], ~i\in[N].
\end{align}
In combination, one can also track unbiased estimates of the central moments (as is the case in Algorithm \ref{Alg:SSEM-UCB-Norm} in Appendix \ref{sec:spa}). A fixed $\mathbf{a}_{t}$ on the other hand eases the empirical moment update and potentially stabilizes the learning progress. Especially $\mathbf{a}_t={\bf 0}$ simplifies derivative evaluation to a single point across all super arms.

Many of our results would also hold for other multivariate polynomial approximation techniques with bounded coefficients. A practical advantage of working with the Taylor series lies in extrapolating the moments for replacement with their approximations, while the partial derivatives allow for pointwise evaluation. Otherwise, the entire function chain would have to be resolved symbolically to identify the occurrences of each power in the payoff, or be subject to another kernel estimation.
\subsection{External UCB}
\label{subsec:ext}
\noindent The graph learning process must be supplied with sufficient explorative data to gain certainty over the structural dependencies. Outside the model's reward simulation, our policies' selection criteria have a confidence addend of
\begin{align}\nonumber
    E_{t}(\mathbf{x}) &= \underset{i\in[N]}{\max}~\mathbf{x}[i]4\mathbf{C}_{t-1}[i]\sqrt{\ln(t-1)}
\end{align}
for $\mathbf{x}\in\mathcal{X}$, which favors base arms with little observation data. Functional uncertainty is thus captured by a single confidence interval around $\mathbf{r}_{\mathbf{x}}(\mathbf{a})$, and required, as derivatives can be arbitrarily close to $0$. Introducing intervals for every derivative value $D^{\boldsymbol{\alpha}}\mathbf{r}_{\mathbf{x}}(\mathbf{a})$ would risk instability or high computation load for tight estimates. Under the assumption of an underlying SEM, the presence of a base arm in the played super arm will already suffice to build up estimation certainty over the associated outward edge weights. The impact of this reward-independent bound can be fine-tuned by multiplication with a hyperparameter. The logarithmic growth ensures the influence of the external exploration on the decision making process over time, independent of scaling factors that would otherwise be critical to surpass the sample order dictated by kernel optimization accuracy guarantees. For short time horizons, an alternative external UCB of $\sum\limits_{i\in[N]}\mathbf{x}[i]4\mathbf{C}_{t-1}[i]\sqrt{\ln(t-1)}/s$ facilitates broader overall information gain, without displacing any $T$-dependent terms in the theoretical bounds of Section \ref{sec:theana}.
\subsection{SSEM Algorithms}\label{sec:Algo}
\noindent All variants of our algorithm are initialized with a series of $\lceil N/s\rceil$ super arms covering the entire base arm set to obtain first estimates of the moments of $\mathcal{B}[i]$, $\forall i\in[N]$. For remaining time steps $t$ up to hidden time horizon $T\in\mathbb{N}$, they then iterate over the mechanics described in the previous subsections. First, $\hat{\mathbf{F}}_{t-1}$ is learned and used to conjecture a reward function approximation $\hat{\mathbf{r}}^{(t-1)}$. Subsequently, the unbiased moment estimates and confidence bounds are updated with the most recent observations. A super arm $\mathbf{x}\in{\mathcal{X}}$ maximizing a weighted sum of the internal and external UCBs,
\begin{align}
\label{eq:criterion}
I_{t}^{\max}(\mathbf{x})+\Delta_{t}^{I}E_t({\mathbf{x}}),
\end{align}
is selected for play, with ties broken at random, and new observations of resulting endo- and exogenous vectors are stored. To strike a balance between the UCBs, we choose the internal span $\Delta_t^{I}=\underset{\mathbf{x}',\mathbf{x}''\in\mathcal{X}}{\max} I_{t}^{\max}(\mathbf{x}')-I_{t}^{\max}(\mathbf{x}'')$ as the dynamic weighting factor. Alternatively, the internal UCB can be limited to an estimate of the maximal true payoff.

The variants differ as follows: SSEM-UCB (Alg. \ref{Alg:SSEM-UCB}) learns graphs from the samples associated with each base arm node individually, and combines them suitable for a Directed Acyclic Graph (DAG).
%
\begin{algorithm}[tb]
\caption{SSEM-UCB: Separable Structural Equation Model - Upper Confidence Bound}
\label{Alg:SSEM-UCB}
\textbf{Input}: Taylor degree $q$, centers $(\mathbf{a}_{t})_{t\in[T]}$; separable RKHSs $(\mathcal{H}^{(i)})_{i\in[N]}$, norm bound $\theta$.
\begin{algorithmic}[1]
\FOR{$t = 1, \dots, \lceil N/s \rceil$}

    \STATE Select $\mathbf{x}_{t}$ with $\mathbf{x}_{t}[i] = 1 \iff (t-1) s < i \leq t s$.

    \STATE Observe $\mathbf{z}_{t}$ and $\mathbf{y}_{t}$.
    
\ENDFOR
\FOR{$t = \lceil N/s \rceil + 1, \dots, T$}
	\FOR{$o = 1, \dots, N$}

        \STATE Obtain $\hat{\mathbf{F}}^{(o)}_{t-1}$ as stated in (\ref{eq:joint}) and (\ref{eq:conv_rest}) but exclusively from feedback during $\{\tau\in[t-1]~|~\mathbf{x}_{\tau}[o]=1\}$.
        
        \STATE Define $\hat{\mathbf{F}}_{t-1}[o-k,o]$ for $k\in[o-1]$ following (\ref{eq:disjoint}) as entries of upper triangular operator $\hat{\mathbf{F}}_{t-1}$.
	\ENDFOR

    \STATE Set $\hat{\mathbf{r}}^{(t-1)} = \mathbf{1}^{\top}(\mathbf{I} - \hat{\mathbf{F}}_{t-1})^{-1}$.

    \STATE Calculate $\mathbf{C}_{t-1}$ and $(\hat{\boldsymbol{\phi}}^{(p)}_{t-1})_{p \in [q]}$ according to (\ref{eq:confidence}), (\ref{eq:moments}).
    
    \STATE Select decision vector $\mathbf{x}_{t}$ that maximizes (\ref{eq:criterion}).
    
    \STATE Observe $\mathbf{z}_{t}$ and $\mathbf{y}_{t}$.
    
\ENDFOR
\end{algorithmic}
\end{algorithm}
%
%
\begin{algorithm}[tb]
\caption{SSEM-UCB-JO: Separable Structural Equation Model - Upper Confidence Bound - Joint Optimization}
\label{Alg:SSEM-UCB-JO}
Exactly as SSEM-UCB, only collapsing the loop over $o\in[N]$ to a single optimization, and exchanging a timed factor in the external UCB.
\begin{algorithmic}[1]
\setcounter{ALC@line}{5}
\STATE Obtain $\hat{\mathbf{F}}_{t-1}$ as stated in (\ref{eq:joint}) and (\ref{eq:conv_rest}).

\setcounter{ALC@line}{11}
\STATE Select decision vector $\mathbf{x}_t$ that instead of (\ref{eq:criterion}) solves
\begin{align}\nonumber
    \underset{\mathbf{x}\in\mathcal{X}}{\arg \max} ~I_{t}^{\max}(\mathbf{x}) + \Delta_{t}^{I}E_{t}(\mathbf{x})\sqrt{\sqrt{t-1}/\ln(t-1)}.
\end{align}
\end{algorithmic}
\end{algorithm}
%
This construction will support our theoretical regret analysis for DAGs, which will argue over this very prediction quality. Less computation intensive and unpresumptuous in structure, SSEM-UCB-JO (Alg. \ref{Alg:SSEM-UCB-JO}) optimizes for a single joint graph per time step, at the cost of necessitating an ordinal increase in $T$ to its external UCB. Due to its greater generality and computational efficiency, we will employ this version for empirical deployment. In Appendix \ref{sec:spa}, we outline an alteration tailored to normal instantaneous reward distributions in particular, alongside a heuristic for choosing the cut-off degree $q$. Further, we analyze algorithmic complexity and potential failure modes in Appendix \ref{sec:per}.
%
%
\section{Theoretical Analysis}
\label{sec:theana}
In this section, we prove upper bounds on the expected regret of the SSEM algorithms. Each theorem closes by listing the supplementary material section of its proof.
\begin{assumption}
\label{asp:rau}
The true operator is realizible, i.e., $\mathbf{F}\in\mathcal{H}_\theta$, and unique, i.e., $(\mathbf{I}-\mathbf{F})^{-1}$ exists.
\end{assumption}
Realizibility can subsequently be foregone for certain choices of dense RKHSs, such as linearized polynomials (see Rem. \ref{rem:unr}). Uniqueness is for example guaranteed for DAGs by backward substitution, and could more generally be relaxed through pseudo-inversion.
\begin{assumption}
\label{asp:fin}
To avoid asymptotic approximations, for Alg. \ref{Alg:SSEM-UCB} we presume the RKHS of $\mathcal{H}$ to be chosen in such a way that the possible entries of $\hat{\mathbf{F}}_{t-1}$ as in (\ref{eq:disjoint}) are from a finite dimensional function space (e.g., polynomials up to a degree).
\end{assumption}

We use the following definitions in our regret analysis. For any decision vector $\mathbf{x} \in \mathcal{X}$, let $\Delta(\mathbf{x}) = \mu(\mathbf{x}^{\ast}) - \mu(\mathbf{x})$ be the \textit{suboptimality gap}. We define $\Delta_{\max} = \underset{\mathbf{x}\in\mathcal{X}}{\max}~ \Delta(\mathbf{x})$, $\Delta_{\min} = \underset{\mathbf{x}\in\mathcal{X}: \mu(\mathbf{x}) < \mu(\mathbf{x}^{\ast})}{\min}~ \Delta(\mathbf{x})$ and $\Delta^{I}=\underset{\tau\in[T]}{\max}\Delta_{\tau}^{I}$. Moreover, let 
{\small
\begin{align*}
	w_{\max} &= \underset{\tau\in[T]}{\max}~\underset{\lvert{\boldsymbol{\alpha}}\rvert\leq q+1}{\max}~\left\lvert D^{{\boldsymbol{\alpha}}}\hat{\mathbf{r}}^{(\tau-1)}_{\mathbf{x}_{\tau}}(\mathbf{a}_{\tau})\right\rvert , \\
	\Gamma^{(q)}_{\max} &= \underset{\tau\in[T]}{\max}\underset{\lvert{\boldsymbol{\gamma}}\rvert=q+1}{\max}\underset{\mathbf{x}\in\mathcal{X}}{\max}\underset{\mathbf{c}\in[0,1]^N}{\max}\sum\limits_{\lvert{\boldsymbol{\alpha}}\rvert=q+1}\left\lvert D^{{\boldsymbol{\gamma}}}\hat{\mathbf{r}}_{\mathbf{x}}^{(\tau-1)}(\mathbf{c})\right\rvert/{\boldsymbol{\alpha}}! , \\
	\psi_{\min}^{(q,s)} &= {\min~\left\{\frac{\min\{q,s\} s^{-\min\{q,s\}}}{w_{\max}(q+1) s}, \frac{1}{\Delta^{I}}\right\}}\frac{\Delta_{\min} - 4\Gamma^{(q)}_{\max}}{8} .
\end{align*}
}
\begin{theorem}
\label{thm:1}
Let $\mathcal{G}$ be a DAG and $\Gamma^{(q)}_{\max} < \Delta_{\min}/4$ for Taylor approximation degree $q$. Under Asm. \ref{asp:rau} and, for Alg. \ref{Alg:SSEM-UCB}, Asm. \ref{asp:fin} , the expected regret of the algorithms is upper bounded by
{\small
\begin{align}\nonumber
    \left[T_{\mathcal{B},\mathcal{H}_{\theta},N} +  \left[ {\frac{4 (s+1) K_T\ln{T}}{ {\psi_{\min}^{(q,s)}}^{2} }} + 1 + \frac{\pi^{2}}{3} (2 q s + 1) \right] N \right] \Delta_{\max},
\end{align}
}%
with $K_T=\ln(T)$ for Alg. \ref{Alg:SSEM-UCB} (SSEM-UCB) and $K_T=\sqrt{T}$ for Alg. \ref{Alg:SSEM-UCB-JO} (SSEM-UCB-JO). [\ref{sec:ProofThm1}]
\end{theorem}
Here, $T_{\mathcal{B},\mathcal{H}_{\theta},N}\in\mathbb{N}$ corresponds to the length of respective initialization phases dictated by the deployed kernel optimization. As our method of choice involves various invariants of $\mathcal{B}$ and $\mathcal{H}$, its value is left for specification in the proof. Dependence on $N$ may also be reallocated to the maximal graph degree and path length (see Supplementary \ref{rem:dep}).  Section \ref{sec:cor} furthermore offers different intuitions through alternative expressions of involved problem invariants. A version of the results from Thm. \ref{thm:1} for a confidence interval length and regret bound in $\mathcal{O}(\ln(T)N)$ is possible, but would build on a priori unknown $\mathcal{B}$-dependent factors for the selection criterion, which could in principle be approximated over time. That regret bounds of this ordinality are close to optimal is apparent from the problem-dependent lower bounds $\Omega(\ln(T)N/\Delta_{\min})$ in \textup{\cite{pmlr-v125-merlis20a}}, which apply even to a priori known reward functions and thereby obsolete endogenous feedback. Conversely, Alg. \ref{Alg:SSEM-UCB-JO} carries less computational expense for optimization and evaluation, and has more general appeal.
\begin{proposition}
\label{pro:1}
For any $\mathcal{G}$, provided $\Gamma^{(q)}_{\max} < \Delta_{\min}/4$ and Asm. \ref{asp:rau}, the time-dependent expected regret of Alg. \ref{Alg:SSEM-UCB-JO} is $\mathcal{O}(N\sqrt{T}\ln(T))$. [\ref{proof:pro:1}]
\end{proposition}
Instead of preconditioning $\Gamma^{(q)}_{\max} < \Delta_{\min}/4$, in Cor. \ref{thm:3} (Appendix \ref{sec:cor}) we restate the regret bound of Thm. \ref{thm:1} for a $4\Gamma^{(q)}_{\max}$-approximate optimal solution. There, we also derive specialized bounds by assuming normally distributed instantaneous rewards (Cor. \ref{cor:1}) or relying on the absence of uncertainty (Cor. \ref{cor:2}). The rationale of Thm. \ref{thm:1} and SSEM-UCB however is robust in the following ways.
\begin{remark}
\label{rem:unr}
To estimate the performance on reward functions outside the solution space described by $\mathcal{H}$, let \[\varphi_{\mathcal{H}_{\theta}} = \underset{i\in[N]}{\max}\underset{\hat{\mathbf{F}}[i]\in\mathcal{H}_{\theta}^{(i)}}{\min}\underset{\mathbf{x}\in\mathcal{X}}{\max} ~\sqrt{\mathbb{E}[(\mathbf{F}-\hat{\mathbf{F}})(\mathbf{y}_{\mathbf{x}}(\mathbf{b}))[i]^2]}\] with $\mathbf{y}_{\mathbf{x}}(\mathbf{b})=\mathbf{y}(\mathbf{x}\odot\mathbf{b})$ denote the optimal obtainable \textit{objective gap}. The proof of Thm. \ref{thm:1} actually tracks this error in $T_{\mathcal{B},\mathcal{H}_{\theta},N}$ and shows that its growth in $q$ is at most exponential. 
For a multivariate linearized polynomial RKHS in $\left\{\mathbb{R}^N\rightarrow\mathbb{R},\mathbf{z}\mapsto\sum\limits_{\substack{j=1\\j\not=i}}^N\sum\limits_{k=1}^{\dim{\mathcal{H}^{(i)}}}r_{j,k}\mathbf{z}[j]^k~{\Bigg|}~r_{j,k}\in\mathbb{R}\right\}$, the selection error can be expressed based on the $q$th Taylor remainder of the payoff as well. Since the remainder bound falls in a factorial of $q$, matching $q$-exponential bounds on the derivatives of $\mathbf{r}$ would suffice to validate Thm. \ref{thm:1} even without assuming realizibility.
\end{remark}
Further, Thm. \ref{thm:1} can easily be adapted to cover multiple types of noise interference, as illustrated next.
\begin{theorem}
\label{thm:4}
Suppose that the exogenous vector is subject to bounded white \emph{observation noise} $\mathbf{n}_t\sim\mathcal{N}$ with $\mathcal{B} + \mathcal{N}$ supported on $[0,1]^N$. Additionally, the SEM may feature unobserved white \emph{model noise} $\mathbf{e}_t\sim\mathcal{M}$ on $\mathbb{R}^N$ so that $\mathbf{y}_t = \mathbf{F}(\mathbf{y}_t) + \mathbf{z}_t + \mathbf{e}_t$, $\forall t\in[T]$. Let $\mathbf{n}_t$, $\mathbf{m}_t$, $\mathbf{b}_t$ be independent across dimensions and among each other. Then, the regret bound of Thm. \ref{thm:1} holds under the same conditions. [\ref{proof:thm:4}]
\end{theorem}
The noise directly affects the range of the optimization target and thus both $w_{\max}$ and $T_{\mathcal{B},\mathcal{H}_{\theta},N}$. The implicit rescaling of the reward function to keep the noisy exogenous observation in $[0,1]^N$ adds to this effect. Variants of Thms. \ref{thm:1} and \ref{thm:4} for subgaussian signals respective noise can be derived from martingale forms of Azuma's inequality \textup{\cite{Azuma67:WSC}} in \ref{subsec:aux}. For endogenous observation noise, see Remark \ref{rem:disbou}.
%
\section{Numerical Analysis}
\label{sec:expana}
\noindent This section is dedicated to the experimental demonstration and performance elaboration of the proposed policy in suitable scenarios. We discuss implementation details and comparisons to a conceptually diverse assortment of benchmark algorithms. Our analysis extrapolates problem quantities indicative for the success and robustness of our approach on noisy synthetic as well as real-world datasets.

\textbf{Benchmarks.} We test SSEM-UCB-JO with multiple choices for the degree of approximation $q$. For its lowest value $1$ and linear function estimates, the method corresponds to a slightly modified version of the SEM-UCB algorithm in 
\textup{\cite{ijcai2022p676}}. Relying on the similar assumption of a separable action space, SGB \textup{\cite{osti_10579433}} accumulates its super arm through gradual inclusion of base arms with the best sample average over a preset number of rounds, before committing and exploiting the final super arm for the remainder of the known time horizon. The structural bandit algorithm METS  \textup{\cite{pmlr-v206-aouali23a}} expects a joint effect prior from which the parameters of action prior distributions and subsequently rewards are sampled. For fairness, we provide it with the distribution type and parameter ranges of the exogenous distributions and have its computation tree form an SEM. For a kernelized bandit comparison, we include IGP-UCB \textup{\cite{pmlr-v70-chowdhury17a}} which calculates reward estimates from kernel matrices between previous and suggested super arms following Gaussian likelihood models with hyperparameters as specified in Section \ref{subsec:ads}.
\subsection{Synthetic Tests}
\label{subsec:synt}
For the synthetic setup we allow super arms of cardinality up to $s=6$ and generate DAG adjacency matrices with an edge density of $0.2$. We conduct trials for $10$ base arms and a mixture of linear and periodic endogenous functions with coefficients in $[0.4, 0.7]$. The exogenous inputs are sampled from normal distributions with means in range $[0.3, 0.7]$ and standard deviations in $[0.1, 0.4]$, rescaled to be supported almost entirely on $[0, 1]$. Two further sets of trials with different distributions, component functions, parameter ranges, or more base arms are compared in \ref{subsec:ads}.

The observations are obfuscated by white noise from standard normal distributions, exogenously scaled by $0.02$, and endogenously by $2~\%$ of the average reward. At each time step, we provide each algorithm with the same exogenous sample masked by its individual super arm choice. The regularization weight of the kernel optimization is tuned adaptively during runtime so that the empirical graphs are appropriately sparse. Without access to density estimates, a computationally more expensive validation split could determine the best value per step. Recommendations towards an efficient implementation are given in Section \ref{subsec:ads}.

We determine the optimal reward empirically and measure the time-averaged regret each algorithm incurs in total up to the current step. In this metric, the algorithms eventually converge to the regret of the super arm they predict to be optimal.
\begin{figure}[!t]
    \centering
    \includegraphics*[width=0.88\columnwidth]{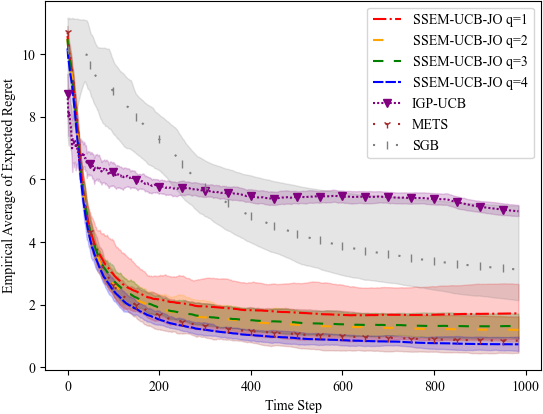}
    \caption{Policies' Regret Confidence Across Runs.}
    \label{fig:synth3}
\end{figure}
%
By incorporating the problem structure into their arm search, the SSEM instances and METS incur much less initial regret than the SGB and IGP-UCB benchmarks and settle for closer to optimal arms sooner. Although the highest derivative approximation of the analytic reward function performs the best among them, the increase is not necessarily steady between degrees, as higher Taylor remainders are not always smaller in absolute expectation.

In Figure \ref{fig:synth3}, only SSEM-UCB-JO of degree $4$ and METS achieved near-zero regret average in time, the latter of which had additional access to bounds on all synthetic distribution parameters. Lower degrees of SSEM-UCB-JO and the remaining benchmarks do not reliably converge to the true optimal arm due to violations of their assumptions. In the case of SSEM-UCB-JO, the more simplistic Taylor approximations are not sufficiently close to the reward series, whereas IGP-UCB and SGB require the expected reward to be in subgaussian distance respective separable. The unstructured benchmarks also experience only the overall effect of the exogenous noise, which affects each super arm differently and in ways not covered by their theoretical guarantees. In addition, the $95\%$ confidence intervals across $10$ runs exhibit a lower diameter and hence greater certainty for SSEM-UCB-JO of higher degrees.
\subsection{Train Delay Dataset}
To demonstrate the application of SSEM-UCB-JO with degree $q=3$ in real-world networks, we work with a dataset on train delays in Germany \cite{train-data}. The set contains data points for nearly $2,000$ stations taken at an hourly rate on $16$ days between 23.06.2024 and 14.07.2024. We extract a local subset of $20$ stops that are at least categorized as urban hubs (category $3$ or less) in geographic locations closest to $10.5^\circ$ N, $51.2^\circ$ E, covering a circular area ca. $95$ km in radius. Each data entry records the deviation of a train from its current schedule with its complete track record. Judging by the direct connections featured on the routes, the physical railway system in the chosen region exhibits a graph density close to $17 ~\%$, with various factors potentially inducing additional associations between stations.

We consider the time difference between planned and actual stay, averaged over all trains at a stop per hour, to be the independent exogenous signal inherent to that node in the transportation network. The train-averaged total delay per station accumulated over the course of each hour is understood as the associated endogenous signal under the complete super arm encompassing the entire set of vertices. Endogenous feedback for other actions is generated according to the SEM computation structure after fitting a graph to the aforementioned data. Projected graph adjacencies at the very least have to include direct connections listed regularly in the track histories and can form cycles. The model is allowed to feature component functions of linear, cubic or (scaled and shifted) hyperbolic tangent variety. Figure \ref{fig:forecast} displays the learned model's progress when reconstructing the overall delay from the immediate delay based on earlier observations.
\begin{figure}[!t]
    \centering
    \includegraphics*[width=0.91\columnwidth]{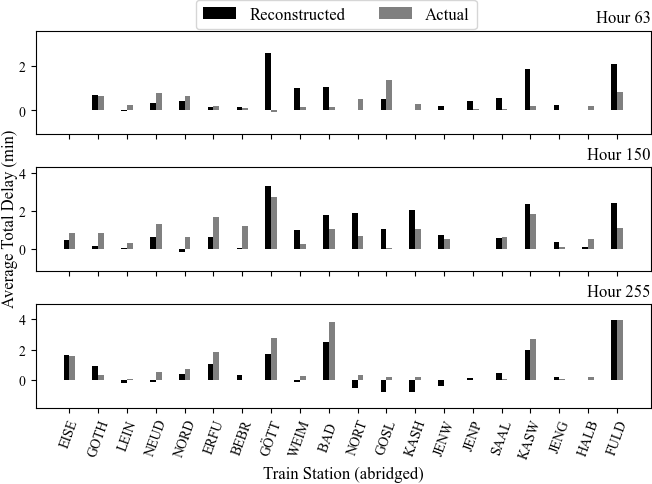}
    \caption{Computational Graph Learning Progress.}
    \label{fig:forecast}
\end{figure}
%
For full names as well as details to the implementation and further visualization, we refer to Section \ref{subsec:adr}. For the agent, we permit super arms of up to $6$ stations with the goal of identifying systemic problems in clusters of the railroad infrastructure. Our simplified test setting does not account for periodic differences between days, nor does it involve the accompanying information on causes of obstruction.

After the initialization phase, we find the rate of change in super arm selection to be generally decreasing, with shifts occurring whenever a better performing combination was explored. The resulting cluster of stations is predicted as chief in overall delay when only selected stations contribute their inherent delay. This subnetwork analysis differs from the apparent set of stations with top overall delay in that it seeks to account for the impact of each station's delay reverberated throughout the system. The final super arm does contain $3$ stations with high inherent delay, and $3$ city stops in close proximity of each other.
%
%
\section{Conclusion}
\label{sec:con}
\noindent By only relying on very elemental properties, the methodologies we introduced are highly customizable through kernel selection and rebalancing of trade-offs befitting specific fields. Expert knowledge can optionally impose additional constraints and guide initial guesses, while the model's causal coherence renders it resilient to a fair amount of inaccuracies and interference. As evidenced by the regret bounds we obtained, in time the algorithms greatly condense the combinatorial intricacy to a linear growth in base arm set size whilst staying sublinear in step count. We exemplified the transfer to data analysis in realistic settings, where they successfully competed against several benchmarks.

Subsequent research could further improve computational scalability through a TS variant that circumvents the determination of optimistic scores by querying instantaneous reward distribution replicas. The parametrization of our model already allows for dimensionality reduction by associating exogenous vectors to multiple nodes and thus reducing the ratio to endogenous outputs. Infinite action spaces can also be covered by contextual or continuously budgeted versions. One way to capture the full reinforcement learning problem beyond bandits would be to extend the fundamental SEM recurrently with time-lagged adjacencies \textup{\cite{topology}}.
%

\newpage
\bibliography{bibliography}
\bibliographystyle{IEEEtran}

\newpage
\appendix

\subsection{COMPARISON}
\label{sec:bac}
In Table \ref{tab:reg}, we detail the time-dependent regret bounds from related papers. Alongside each, we list a few key assumptions in which they differ from the work conducted here. Some framworks defer the regret-impact of base arm number $N$ to graph properties such as the maximal degree $d$ and maximal path length $L$, which without further assumptions can be of order $N$. While frameworks and semi-bandit observations vary, intervened node values or prior contexts may be compared to filtered feedback by the volume of data available for learning. Compared to general SEMs, separable SEMs are less computational expansive for high-dimensional data, carry less risk of overfitting, and interactions associated to each edge are more easily interpretable. For nonlinear component functions, the inverse $(\mathbf{I}-\mathbf{F})^{-1}$ still captures more complex interactions, as it is in general not separable, nor are equilibrium solutions of cyclic SEMs.
\bgroup
\def\arraystretch{1.5}
\begin{table*}[!ht]
\caption{Ordinal Comparison of Regret Bounds}
\label{tab:reg}
\begin{center}
\begin{tabular}{|l|l|l|}
\hline
Work & Regret & Key Assumptions \\
\hline
\hline
\cite{ijcai2022p676} & $\mathcal{O}(N\ln T)$ & Linear SEM, noiseless observations, DAG \\
\hline	
\cite{NIPS2016_aa169b49} & $\mathcal{O}(\sqrt{NT\ln T})$ & Oracle access, monotone reward function \\
\hline	
\cite{pmlr-v70-chowdhury17a} & \makecell*[cl]{$\mathcal{O}(\sqrt{T\gamma_T(\gamma_T+\ln(1/\delta))})$,\\where $\gamma_T=\mathcal{O}((\ln T)^{N+1})$\\for squared exponential kernels} & Expected scores observed with homogeneous noise across arms \\
\hline	
\cite{pmlr-v202-hwang23a} & \makecell*[cl]{$\tilde{\mathcal{O}}\left(\sqrt{\tilde{d}\max\{\tilde{d},s\}T}\right)$,\\where $\tilde{d}$ is the effective\\neural tangent dimension} & Contexts encode arm dependencies, monotone score accumulation \\
\hline	
\cite{pmlr-v238-yan24a} & $\mathcal{O}(d^{L-1}\sqrt{T\ln(NT)})$ & Free selection of ingoing signal values, no action cardinality restriction, known DAG structure \\
\hline	
\cite{Feng_Chen_2023} & $\mathcal{O}(N\sqrt{dT}\log T)$ & Monotonic generalized linear functions, endogenous interventional control, known DAG structure \\
\hline	
\cite{JMLR:v24:22-0969} & \makecell*[cl]{$\tilde{\mathcal{O}}(\beta_T^{L+1}(d+1)^{L/2}\sqrt{NT})$,\\where $\beta_T=\tilde{\mathcal{O}}(d\sqrt{\ln(NT)})$} & Linear SEM, DAG \\
\hline	
\cite{NEURIPS2024_2aba6ec2} & $\tilde{\mathcal{O}}(d^{L-0.5}\sqrt{T})$ & Linear SEM, known graph degree, DAG \\
\hline	
\cite{osti_10579433} & $\mathcal{O}((\sqrt{N}T\ln T)^{2/3})$ & Submodular and monotonic expected reward \\
\hline	
\cite{pmlr-v206-aouali23a} & $\tilde{\mathcal{O}}(\sqrt{NT})$ & Exemplary nonlinearities only for final accumulation, DAG \\
\hline	
\cite{NEURIPS2024_3640a199} & $\tilde{\mathcal{O}}(\sqrt{NT})$ & No nonlinear transformation after action selection \\
\hline	
\hspace{0mm}[ours] & \makecell*[cl]{$\mathcal{O}(N\ln^2(T))$ for Alg. \ref{Alg:SSEM-UCB},\\$\mathcal{O}(N\sqrt{T}\ln(T))$ for Alg. \ref{Alg:SSEM-UCB-JO}} & Separable SEM, DAG for Alg.  \ref{Alg:SSEM-UCB} \\
\hline
\end{tabular}
\end{center}
\end{table*}
\egroup

\subsection{ADAPTATIONS}
\label{sec:ada}
\subsubsection{Extensions}
\label{sec:ext}
For simplicity, we fixed $\mathbf{G}=\mathbf{I}$ for the preliminary on-node activations. Generic exogenous transformations could be identified by optimization together with the entries of $\mathbf{F}$, effectively leading to regret bounds as if another base arm were involved. Our proofs will still specify the step in which this identification would take place.

Besides, we focused on zero-true $\mathbf{F}$. Otherwise, non-zero $\mathbf{F}_0: [0,1]^N\rightarrow\mathbb{R}^N, \mathbf{b}\mapsto(\sum\limits_{j=1}^N\mathbf{F}[i,j](0))_{i\in[N]}^{\top}$ can be linearly transformed as $\mathbf{I} - \mathbf{F} + \mathbf{F}_{0} = (\mathbf{I} +\mathbf{F}_{0})(\mathbf{I} - \mathbf{F})$ towards a zero-true operator $\mathbf{F}-\mathbf{F}_0$ covered in our analysis. Uncertainty over the entries of $\mathbf{F}_0$ would have the same effect on the regret bounds as another exogenous operator since $(\mathbf{I} - \mathbf{F})^{-1} = (\mathbf{I} - (\mathbf{F} - \mathbf{F}_{0}))^{-1} (\mathbf{I} +\mathbf{F}_{0})$, assuming the inverses exist.

Although we accumulated the final payoff from the overall rewards in a strictly linear manner, one could frame it through a function operator $\mathbf{h}$ of dimensions $(N,1)$, i.e., $\mathbf{r}_{\mathbf{x}_t}(\mathbf{b}_t) = \mathbf{h}^{\top}\mathbf{y}_t = \sum_{i\in[N]}\mathbf{h}[i](\mathbf{y}[i])$. If $\mathbf{h}$ is known, it is straightforward to combine our theoretical results with known bounds for the expectation of the accumulating functions; Otherwise, their learning process would require the payoff as additional feedback and increase $T_{\mathcal{B},\mathcal{H}_{\theta},N}$ according to the complexity of the candidate function space in the regret bounds of Section \ref{sec:theana}.

\begin{remark}
i) Setting $\mathbf{r}_{\mathbf{x}}(\mathbf{b}) = \mathbf{r}(\textup{diag}(\mathbf{x})\mathbf{b})$ implies that only the instantaneous rewards of the selected arms impact the reward. If the selection only altered the signal strength, e.g., $\mathbf{r}_{\mathbf{x}}(\mathbf{b}) = \mathbf{r}((\textup{diag}(\mathbf{x})+k\mathbf{I})\mathbf{b})$, $k\in\mathbb{R}$, we would need to treat $s$ as $N$ for most of our theoretical analysis but drastically reduce $T_{\mathcal{B},\mathcal{H}_{\theta},N}$, as every base distribution would be involved no matter the agent's decision. Separate sets of arms completely outside the agent's control but not perception could be treated similarly.

ii) In our setup, the agent selects exogenously, as its decision vector masks the independent instantaneous rewards and their observability directly. Were we to restate the problem by shifting the mask to the associated endogenous vectors, i.e., \[\mathbf{r}_{\mathbf{x}_t}(\mathbf{b}_t) = {\bf 1}^{\top}(\mathbf{x}_t\odot\mathbf{y}_t) = {\bf 1}^{\top}(\mathbf{x}_t\odot(\mathbf{I}-\mathbf{F})^{-1}(\mathbf{G}(\mathbf{b}_t))),\] we could replicate the analyses of this paper with appropriate adjustments to the selection strategies and only minor changes in theoretical variables. In either selection scenario, if both vectors' observability is restricted to the chosen nodes, the external UCB should track instances of pairwise base arm occurrence, leading to a replacement of $N$ with approximately $N^2$ in regret bounds. Partially masked feedback vectors could also be covered by the objective gap.

iii) Alternative uniquely determined SEMs of the form $\mathbf{y} = \mathbf{F}(\mathbf{y}+\mathbf{z})$ require only adjustments of signal ranges as they are equivalent to $\mathbf{y} = (\mathbf{I}-\mathbf{F})^{-1}(\mathbf{z})-\mathbf{z}$. Structural input to the agent at each time step, i.e. contexts, for example as exogenous addends, likewise would not affect ordinality.
\end{remark}

\begin{remark}
\label{rem:disbou}
i) Were we to forego the lower bounds to the distributions $\mathcal{B}_i$, the continuity of their p.d.f.s would allow for an asymptotic argument in expected difference to a truncation from below. This argument would validate the regret bound of Theorem \ref{thm:1} with another invariant of $\mathcal{B}$ involved in the calculation of $T_{\mathcal{B},\mathcal{H}_{\theta},N}$. Moreover, discrete distributions would require the RKHS to be of sufficiently low dimension to determine component functions from occupied values.

ii) Ordinally, Theorem \ref{thm:4} applies to i.i.d. samples from white endogenous $\mathcal{N}_2$ instead of exogenous observation noise as well. In conjunction, bounds on the conditional denoising error $\mathbb{E}_{\substack{\mathbf{b}\sim\mathcal{B}\\\boldsymbol{\epsilon}\sim\mathcal{N}_2}}[\lvert\mathbf{F}(\mathbf{y}(\mathbf{b})) - \mathbb{E}_{\substack{\tilde{\mathbf{b}}\sim\mathcal{B}\\\tilde{\boldsymbol{\epsilon}}\sim\mathcal{N}_2}}[\mathbf{F}(\mathbf{y}(\tilde{\mathbf{b}}))~|~\mathbf{y}(\tilde{\mathbf{b}})+\tilde{\boldsymbol{\epsilon}}=\mathbf{y}(\mathbf{b})+\boldsymbol{\epsilon}]\rvert]$ can be utilized in the general regret analysis without imposing additional properties on the distributions or the functions.
\end{remark}

One could seamlessly employ the subroutines of \textup{\cite{piecewiseSEM}} or \textup{\cite{delayedSEM}} to better cope with structurally nonstationary environments, shifting distributions or feedback delays.

Besides its expressiveness, our choice of causal model class mostly pertains to its sample efficient identification, and only the related parts would require major overhauls for different representations. In particular, general bounds can be derived from our results for any analytic objective function if known to the agent. Other signal transformations, such as graph neural networks, could be combined with matching confidence schemes (e.g., neural bandit \textup{\cite{pmlr-v202-hwang23a}}) to encompass an even wider range of problem setups categorically. On a DAG, the SEM already resembles the forward propagation of a neural network that outputs all its intermediate node values, with the instantaneous rewards acting as biases and function activations associated to edges rather than nodes.

Though not studied here, ideally the external UCB would take the postulated graph adjacencies into account when evaluating the impact of optimization certainty improvements gained from inclusion of each base arm, reminiscent of probabilistic trigger settings \textup{\cite{pmlr-v202-kocak23a}}. From a Bayesian perspective, replacement with a more targeted parameter information-directed sampling design \textup{\cite{NIPS2014_90720a2f}} could aid the graph learning.
%
\subsubsection{Specializations}
\label{sec:spa}
Tailored to normal instantaneous reward distributions in particular, SSEM-UCB-Norm (Algorithm \ref{Alg:SSEM-UCB-Norm}) decreases the number of moments that go into the internal UCB calculation to two regardless of $q$.
%
\begin{algorithm}[tb]
\caption{SSEM-UCB-Norm: Separable Structural Equation Model - Upper Confidence Bound - Normal}
\label{Alg:SSEM-UCB-Norm}
Exactly as SSEM-UCB, but running the initialization loop twice, and replacing empirical moment calculation and selection criterion as follows.
\begin{algorithmic}[1]

    \STATE Calculate $\mathbf{C}_{t-1}$, $\hat{\boldsymbol{\beta}}_{t-1}=\mathbf{a}_t$ and $\hat{\boldsymbol{\phi}}^{(2)}_{t-1}$ according to (\ref{eq:confidence}), (\ref{eq:means}) and (\ref{eq:moments}), respectively, and set $\tilde{\mathbf{m}}_{t-1}[j] = \frac{\mathbf{m}_{t-1}[j]}{\mathbf{m}_{t-1}[j]-1}$ for $j\in[N]$.

    \STATE Select decision vector $\mathbf{x}_{t}$ that solves
    \begin{align}\nonumber
        &\underset{\mathbf{x}\in\mathcal{X}}{\arg \max} ~E_{t}(\mathbf{x}) +\underset{\substack{\mathcal{C}_{t-1}\in \prod_{i\in[N]}[-\mathbf{C}_{t-1}[i], \mathbf{C}_{t-1}[i]] \\ \mathcal{W}^{(2)}_{t-1}\in \prod_{i\in[N]}\{-\mathbf{C}_{t-1}[i], \mathbf{C}_{t-1}[i]\}}}{\max} \\ \nonumber
        &\sum_{\substack{\lvert{\boldsymbol{\alpha}}\rvert\leq q\\{\boldsymbol{\alpha}} \text{ mod } 2 = 0}}  \frac{D^{{\boldsymbol{\alpha}}}\hat{\mathbf{r}}^{(t-1)}(\mathbf{x}\odot(\hat{\boldsymbol{\beta}}_{t-1} + \mathcal{C}_{t-1}))}{{\boldsymbol{\alpha}}!!} \\ \nonumber
        &\hspace{-2mm}\prod_{j\in\mathcal{I}(\mathbf{x})}\left(\tilde{\mathbf{m}}_{t-1}[j]{\hat{\boldsymbol{\phi}}^{(2)}_{t-1}}[j] + \sqrt{\tilde{\mathbf{m}}_{t-1}[j] }\mathcal{W}^{(2)}_{t-1}[j]\right)^{{\boldsymbol{\alpha}}[j]/2}.
    \end{align}
\end{algorithmic}
\end{algorithm}
%
In turn, it also needs to apply its confidence to the derivatives' argument. The maximization over the confidence interval is well-defined due to the continuity of the underlying functions and their derivatives.

Since the distance of the deployed Taylor series approximation to the actual reward function is problem-dependent, the chosen cutoff degree $q$ can be gradually increased over the course of the algorithm's runtime until the diminishing analytical remainder falls within bounds proportional to the minimal empirical payoff discrepancy of suboptimal super arms, as outlined in Algorithm \ref{Alg:q-search}.
\begin{algorithm}[tb]
\caption{q-search}
\label{Alg:q-search}
\textbf{Input}: Time step $t$, input of SSEM-UCB except for $q$.\\
\begin{algorithmic}[1]
\STATE Initialize $q = -1$, $\Gamma=0$, $\hat{\Delta}_{\min} = 0$.
\WHILE{$\Gamma \geq \hat{\Delta}_{\min}/4$ \AND $q < t$}
    
    \STATE Increase $q \leftarrow q + 1$.
    \STATE Calculate $\hat{\Delta}_{\min} \leftarrow \underset{\substack{\mathbf{x}\in\mathcal{X}\\S_{\max}-S(\mathbf{x})>0}}{\min}S_{\max} - S(\mathbf{x})$ for $S_{\max} \leftarrow \underset{\mathbf{x}'\in\mathcal{X}}{\max} ~S(\mathbf{x}')$ based on the selection strategy $S$ and reward estimate $\hat{\mathbf{r}}^{(t-1)}$ of SSEM-UCB with parameter $q$.
    \STATE Set $\Gamma$ to a $q$-remainder bound of $\hat{\mathbf{r}}^{(t-1)}$ (either as provided in \ref{subsec:int} or as an empirical estimate).

\ENDWHILE
\STATE {\bf return} $q$
\end{algorithmic}
\end{algorithm}
%
\subsection{PERFORMANCE}
\label{sec:per}
{\bf Complexity.} The computational expense of each time step $t$ is split between kernel optimization, differentiation and arm comparison. The former two strongly depend on the method employed. For example, the Alternating Direction Method of Multipliers (ADMM) identifies the best sample fit incrementally in $\mathcal{O}(t^2N)$ for Alg. \ref{Alg:SSEM-UCB-JO} and $\mathcal{O}(\sum_{i\in[N]}\mathbf{x}_t[i]\mathbf{m}_t[i]^2N)$ for Alg. \ref{Alg:SSEM-UCB} \textup{\cite{8831393}}. For $\mathbf{a}_t={\bf 0}$, very basic numerical differentiation can be performed in $\mathcal{O}(w~\text{eval}(\hat{\mathbf{F}}_{t-1}))$, where $w=(qN)^{\min\{s,q/2\}}$ and $\text{eval}(\hat{\mathbf{F}}_{t-1})$ is the complexity of evaluating $\hat{\mathbf{F}}_{t-1}$. With all partial derivative values in memory, the optimal combinatorial action amongst $\lvert\mathcal{X}\rvert=\sum_{k=0}^s\binom{N}{k}$ can always be picked in $\mathcal{O}(w+N^ss^2)$ by forming a computation tree of viable combinations. These most general bounds can be drastically decreased by graph sparsity (enforced through regularization or edge recognition thresholds), optimized block matrix inversion, refitting the graph on demand or in parallel, exploiting further structure or restrictions on the action space, randomizing challengers to the leading arm, and other adjustments described in Sections \ref{sec:expana} and \ref{sec:expadd}.

{\bf Failure modes.} Since the external UCB ensures functional certainty through minimal base arm count, if all super arms containing a specific base arm perform exceptionally subpar, its cost might become unduly dominant, as is evidenced by the regret bounds linearity in $\Delta_{\max}$. A successive arm elimination subroutine, where discrepancies in experienced performance are compared to a threshold for further inclusion, helps the algorithms to cope with this scenario.
When dealing with very complex or extensive graphs, unless there are substantial risk factors, it is also advisable to spend additional iterations on the initial pure exploration phase to build up sufficient structural knowledge sooner.
Further, for problems in which the operator $\mathbf{I}-\mathbf{F}$ is singular, the solution to the SEM is not unique and structural learning instable. Specifying endogenous ranges, or secondary constraints such as endogenous norm minimization, can facilitate stable pseudo-inversion of the SEM while reducing the computational complexity of exact inversion. Greater detail on such generalizations and their integration into the algorithms is left to future work.
%
\subsection{COROLLARIES}
\label{sec:cor}
An alternative expression for the regret bound of Theorem \ref{thm:1} presents itself through Cauchy's estimate and Hartogs's theorem on the analytic continuations $\overline{\mathbf{r}}^{\tau}$, $\tau\in[T]$, of the empiric reward functions to the complex numbers. For any $l\in\mathbb{R}_{>0}$ within the radius of convergence, we thereby have $\underset{\tau\in[T]}{\max}~\sum\limits_{\lvert{\boldsymbol{\alpha}}\rvert\leq q} \lvert D^{{\boldsymbol{\alpha}}}\overline{\mathbf{r}}^{(\tau-1)}_{\mathbf{x}_{\tau}}(\mathbf{a}_{\tau})\rvert / {\boldsymbol{\alpha}}! \leq \underset{\tau\in[T]}{\max}~\sum\limits_{\lvert{\boldsymbol{\alpha}}\rvert\leq q} \underset{\mathbf{P}_{l}(\mathbf{a}_{\tau})}{\sup}\lvert\overline{\mathbf{r}}^{(\tau-1)}\rvert / l^{\lvert{\boldsymbol{\alpha}}\rvert}$, where $\mathbf{P}_{l}(\mathbf{a}_{\tau})$ is the polydisc of radius $l$ in all directions around $\mathbf{a}_{\tau}$ in $\mathbb{C}$. For example, for $l = s+q-1$ we get $e~\underset{\tau\in[T]}{\max}~\underset{\mathbf{P}_{s+q-1}(\mathbf{a}_{\tau})}{\sup}\lvert\overline{\mathbf{r}}^{(\tau-1)}\rvert$ as a limit to the absolute sum. Multiplied by $(\Delta_{\min}-4\Gamma_{\max}^{(q)})/8$, this term can be substituted for $\psi_{\min}^{(q,s)}$ in Theorem \ref{thm:1}.

For an approximate formulation of Theorem \ref{thm:1}, we set $\Delta_{\max}^{(q)} = \underset{\mathbf{x}: \mu(\mathbf{x}) + 4\Gamma^{(q)}_{\max} < \mu(\mathbf{x}^{\ast})}{\max}~ \Delta(\mathbf{x})$ and $\Delta_{\min}^{(q)} = \underset{\mathbf{x}: \mu(\mathbf{x}) + 4\Gamma^{(q)}_{\max} < \mu(\mathbf{x}^{\ast})}{\min}~ \Delta(\mathbf{x}) \geq 4\Gamma^{(q)}_{\max}$.
\begin{corollary}
\label{thm:3}
Provided that $\mathbf{F}[i]\in\mathcal{H}_{\theta}^{(i)}~\forall~i\in[N]$, then, for any degree $q\in\mathbb{N}$, the expected regret of the SSEM-UCB algorithm is upper bounded as
\begin{align}\nonumber
  &\left[ {\frac{4 (s+1) \ln^{2}{T}}{ {\psi_{\min}^{(q,s)}}^{2} }} + 1 + \frac{\pi^{2}}{3} (2 q s + 1) \right] N \Delta_{\max}^{(q)} \\\nonumber
   &\hspace{5mm} + T_{\mathcal{B},\mathcal{H}_{\theta},N} \Delta_{\max}^{(q)} + 4\Gamma^{(q)}_{\max} T.
\end{align}
Proof. \textup{See Section \ref{proof:thm:3}.} \hfill $\square$
\end{corollary}
Next, we state the regret bound associated with the specialized Algorithm \ref{Alg:SSEM-UCB-Norm} in Section \ref{sec:spa}.
%
\begin{corollary}
\label{cor:1}
In the setting of Theorem \ref{thm:1}, if the instantaneous reward distributions $\mathcal{B}_i$ are all normal and all $T$ samples $\mathbf{b}_t$ lie in $[0,0.5]^N$, the expected regret of the SSEM-UCB-Norm algorithm is upper bounded as
\begin{align}\nonumber
   &\left[ {\frac{9 (q-1)!!q (s+1) \ln^{2}{T}}{{ \psi^{(q,s)}_{\min}}^{2} }} + 2 + \frac{\pi^{2}}{3} (4 s + 1) \right] N \Delta_{\max} \\\nonumber
   &\hspace{5mm}+ T_{\mathcal{B},\mathcal{H}_{\theta},N} \Delta_{\max}.
\end{align}
Proof. \textup{See Section \ref{proof:cor:1}.} \hfill $\square$
\end{corollary}
Introducing noise would lead to results analogous to Theorem \ref{thm:4}. The benefits of less averages to track lie in memory usage as well as computation effort required to identify each super arm's most optimistic conjecture. Their impact on performance would be more noticeable in Bayesian regret analysis than the absolute guarantees studied here, which necessarily involve many worst-case problem instance evaluations. A Bayesian view would likely improve all acquired bounds in their ordinality, especially for sparse graphs, which are the recommended application of SSEM-UCB due to the computation load associated with long paths and, to a lesser extent, high degrees.
%
\begin{corollary}
\label{cor:2}
Let $\mathbf{M}\in[0,1]^{N\times N}$ be an invertible upper triangular matrix. In the idealized setting of Theorem \ref{thm:1}, an initialization phase to SSEM-UCB that plays the $i$th column vector of $\mathbf{M}$ a number $\dim(\mathcal{H}^{(i)})$ times leads to a regret bound of
{\small
\begin{align}\nonumber
	\left[\sum\limits_{i\in[N]}\dim(\mathcal{H}^{(i)}) + \left[ {\frac{4 (s+1) \ln{T}}{ {\psi_{\min}^{(q,s)}}^{2} }} + 1 + \frac{2\pi^{2}}{3} q s \right] N\right] \Delta_{\max}
\end{align}
}%
without the term $-4\Gamma^{(q)}_{\max}$ in $\psi_{\min}^{(q,s)}$.\\
Proof. \textup{See Section \ref{proof:cor:2}.} \hfill $\square$
\end{corollary}
In contrast to Theorem \ref{thm:4}, the preceding corollary is valid only in the absence of noise or when the graph is known.
%
\subsection{PROOF OF THEOREM IV.3}
\label{sec:ProofThm1}
%
\subsubsection{Notations}
\label{subsec:Notation}
Before proceeding to the proof itself, we introduce some important notations together with their definitions.
We define the \textit{index set} of a decision vector $\mathbf{x} \in \mathcal{X}$ as $\mathcal{I}(\mathbf{x}) = \left\{ i ~|~ \mathbf{x}[i] \neq 0, \forall i \in [N] \right\}$. The confidence bound of base arm $i$ at time $t$ was declared as $\mathbf{C}_{t}[i] = \sqrt{\frac{(s+1) \ln{t}}{\mathbf{m}_{t}[i]}}$. At each time $t$, we collect the unbiased empirical estimates of moments $\hat{\boldsymbol{\phi}}^{(k)}_{t}[i]$ about the center point $\mathbf{a}_{t+1}\in[0,1]^N$ and the calculated confidence bounds $\mathbf{C}_{t}[i]$ of all base arms $i \in [N]$ in vectors $(\hat{\boldsymbol{\phi}}^{(k)}_{t})_{k\in[q]}$ and $\mathbf{C}_{t}$, respectively. We write $\boldsymbol{\phi}_t^{(k)}$ for the actual moments about $\mathbf{a}_{t+1}$ of the instantaneous rewards. For time step $t+1$, with $k_p = 2^{2\text{ mod }p}$, we set the \textit{internal selection index} for a decision vector $\mathbf{x} \in \mathcal{X}$ to
\begin{align}\nonumber
	&I_{t+1}^{\max}(\mathbf{x})= \\ \nonumber
	&\hspace{2mm}\underset{\mathcal{W}^{(p)}_{t}\in \prod\limits_{i\in[N]}\{\max\{-k_p\mathbf{C}_{t}[i],1-k_p-\hat{\boldsymbol{\phi}}^{(m)}_t\}, \min\{k_p\mathbf{C}_{t}[i],1-\hat{\boldsymbol{\phi}}^{(m)}_t\}\}}{\max} \\ \nonumber
        &\hspace{5mm}\sum_{\lvert{\boldsymbol{\alpha}}\rvert\leq q} \frac{D^{{\boldsymbol{\alpha}}}\hat{\mathbf{r}}^{(t)}_{\mathbf{x}}(\mathbf{a}_{t+1})}{{\boldsymbol{\alpha}}!} \prod_{j\in\mathcal{I}(\mathbf{x})}\left(\hat{\boldsymbol{\phi}}^{({\boldsymbol{\alpha}}[j])}_{t}[j] + \mathcal{W}^{({\boldsymbol{\alpha}}[j])}_{t}[j]\right)
\end{align}
and $I^{\min}_{t+1}(\mathbf{x})$ analogously. Here, we w.l.o.g. restricted $\hat{\boldsymbol{\phi}}^{(k)}_{t}[i] + \mathcal{W}^{(k)}_{t}[i]$ to the range $[-1,1]$ of the actual moments.
Let $\mathcal{W}^{(p)_{\max}}_{t-1}(\mathbf{x})$ and $-\mathcal{W}^{(p)_{\min}}_{t-1}(\mathbf{x})$, $p\in[q]$, be confidence scores that realize $I^{\max}_t(\mathbf{x})$ and $I^{\min}_t(\mathbf{x})$ respectively. We call
\begin{align}\nonumber
    E_{t+1}(\mathbf{x}) = \Delta_{t+1}^{I}\underset{i\in[N]}{\max}~\mathbf{x}[i]4\mathbf{C}_{t}[i]\sqrt{\ln(t)}
\end{align}
the (weighted) \textit{external selection index}.

At time $t$, the algorithm selects a decision vector $\mathbf{x}_t$ that maximizes $I^{\max}_t+E_{t}$. For ease of presentation, in the sequel, we use the equivalence ${\bf 1}^{\top} (\mathbf{I} - \hat{\mathbf{F}}_{t-1})^{-1} \hat{\mathbf{G}}_{t-1}(\text{diag}(\mathbf{b}) \mathbf{x}_{t}) = {\bf 1}^{\top} (\mathbf{I} - \hat{\mathbf{F}}_{t-1})^{-1} \hat{\mathbf{G}}_{t-1}\text{diag}(\mathbf{x}_{t}\cdot\text{id}) (\mathbf{b})$ for all $\mathbf{b}\in[0,1]^N$.
To simplify the notation, sometimes we drop the time index $t$ in $\mathbf{m}_{t}[i]$ and use $\mathbf{m}[i]$ to denote the number of times that the base arm $i$ has been observed up to the current time instance. Since the choice of center point for the Taylor series expansion at each time step is free, we also use $\mathbf{a}$ and $\boldsymbol{\phi}^{(k)}$ in lieu of $\mathbf{a}_{t+1}$ and $\boldsymbol{\phi}^{(k)}_{t}$, respectively.

For any $\mathbf{x} \in \mathcal{X}$, we use the counter $\mathcal{T}_{\mathbf{x}}(t)$ to represent the total number of times the decision vector $\mathbf{x}$ has been selected up to time $t$. Finally, for each base arm $i \in [N]$, we define a counter $\mathscr{T}_{i}(t)$ which is updated as follows. At each time $t > T_{\mathcal{B},\mathcal{H}_{\theta},N}$, after a starting point $T_{\mathcal{B},\mathcal{H}_{\theta},N}\in\mathbb{N}$ to be specified later on, that a suboptimal decision vector $\mathbf{x}_{t}$ is selected, we have at least one base arm $i \in [N]$ such that $i = \underset{i \in \mathcal{I}(\mathbf{x}_{t})}{\arg\min} ~\mathbf{m}_{t}[i]$. In this case, if the base arm $i$ is unique, we increment $\mathscr{T}_{i}(t)$ by 1. If there is more than one such base arm, we break the tie and select one of them arbitrarily to increment its corresponding counter.

For a distribution $\mathcal{D}$ on a Lebesgue measurable space $(S,\Sigma,\mu)$ and analytic functions $h, g: S\rightarrow\mathbb{R}$, we call
\begin{align}\nonumber
	d_{\mathcal{D}}(h,g) &= \sqrt{\mathbb{E}_{\mathbf{d}\sim\mathcal{D}}[(h(\mathbf{d})-g(\mathbf{d}))^2]} \\\nonumber
	&= \sqrt{\int_{S} \text{p.d.f.}_{\mathcal{D}}(\mathbf{d}) (h-g)(\mathbf{d})^2 d\mathbf{d}}
\end{align}
the distance between $h$ and $g$ respective $\mathcal{D}$. We note that on finite dimensional spaces of analytic functions and for $\mu(\text{supp}(\mathcal{D}))>0$, the corresponding weighted $L^2$-seminorm is a true norm, because the intersection of distinct analytic functions has measure $0$. In principle, this would allow for norm equivalence arguments in the optimization part of the proof. Since we need to track dependencies on $T$ or $N$ however, we will state equivalence factors explicitly.  
\subsubsection{Auxiliary Results }
\label{subsec:aux}
We use the following lemmata in the proof of Theorem 1.
\begin{lemma}
{\textup{\cite{Azuma67:WSC}}}
\label{lem:Hoeffding}
Let $z_{1}, z_{2}, \dots, z_{m}$ be random variables, $c,d\in \mathbb{R}$ and $z_{i} \in [c,d]$, $\forall i$. Moreover, $\mathbb{E}[z_{t} | z_{1}, \dots, z_{t-1}] = \gamma$, for all $t = 1, \dots, m$. Then, for all $D \geq 0$,
\begin{equation}
     \mathbb{P} {\Bigg[} {\Big |} \sum_{i = 1}^{m} z_{i} - m \gamma {\Big |} \geq D {\Bigg]} \leq 2 e^{- \frac{2 D^{2}}{m (d-c)^2} }.
\end{equation}
\end{lemma}
\begin{lemma}
\label{lem:dish}
Let $\mathcal{H}\subseteq L^2(S,\mu)$ be a vector space of analytic functions and $h\in\mathcal{H}$ map from the set of a Lebesgue measure space $(S,\Sigma,\mu)$ to $\mathbb{R}$. Let $\mathcal{B}_1$ and $\mathcal{B}_2$ be distributions on subsets $S_1$ and $S_2$ of $S$, respectively. Further, let $l_{\mathcal{B}_1}$ be a lower bound to the existing p.d.f. of $\mathcal{B}_1$, $\mu(S_1)>0$ and $(e_1, \dots, e_l)$ be an orthonormal basis of $\mathcal{H}$ relative $S_1$. Then $\mathbb{E}_{\mathcal{B}_2}[h^2] \leq \mathbb{E}_{\mathcal{B}_1}[h^2](\sum_{i=1}^{l}\sqrt{\mathbb{E}_{\mathcal{B}_2}[e_i^2]})^2/l_{\mathcal{B}_1}$ holds true.
\end{lemma}
\textit{Proof.} For the basis representation $h=\sum_{i=1}^{l}k_ie_i$ we have
\begin{align}\nonumber
	 \mathbb{E}_{\mathcal{B}_1}[h(b)^2] &\stackrel{(a)}{=} \int_{S_1}\left(\sum_{i=1}^{l}k_ie_i(b)\right)^2\text{p.d.f.}_{\mathcal{B}_1}(b)db \\\nonumber
	 &\geq l_{\mathcal{B}_1}\int_{S_1}\left(\sum_{i=1}^{l}k_ie_i(b)\right)^2db \\\label{eq:maxnorm}
	 &\stackrel{(c)}{=}l_{\mathcal{B}_1}\sum_{i=1}^{l}k_i^2 \geq l_{\mathcal{B}_1}~{\underset{i\in[l]}{\max}~k_i^2},
\end{align}
where $(c)$ is due to the orthonormality of the basis over $S_1$ and $(a)$ obeys the law of the unconscious statistician. This bound to the maximum norm allows us to conclude
\begin{align}\nonumber
	\sqrt{\mathbb{E}_{\mathcal{B}_2}[h(b)^2]} &\stackrel{(a)}{\leq} \sum_{i=1}^{l}\lvert k_i\rvert\sqrt{\mathbb{E}_{\mathcal{B}_2}[e_i(b)^2]} \\\nonumber
	&\stackrel{(\ref{eq:maxnorm})}{\leq} \sqrt{\frac{1}{l_{\mathcal{B}_1}} \mathbb{E}_{\mathcal{B}_1}[h(b)^2]}\sum_{i=1}^{l}\sqrt{\mathbb{E}_{\mathcal{B}_2}[e_i(b)^2]},
\end{align}
with $(a)$ derived from the Minkowski inequality. \hfill $\square$

Lastly, we note that, since summation and concatenation preserve the analytic property, it is maintained by the vector functions $\mathbf{I} - \hat{\mathbf{F}}_{t}$ on $[0,1]^N$. Their Jacobians are upper triangular matrices with ones on their diagonal and thus of full rank. By Lagrange's inversion theorem, the analytic property propagates to the inverses $(\mathbf{I} - \hat{\mathbf{F}}_{t})^{-1}$ and hence the empiric reward functions $\hat{\mathbf{r}}^{(t)}_{\mathbf{x}} = {\bf 1}^{\top} (\mathbf{I} - \hat{\mathbf{F}}_{t})^{-1} \hat{\mathbf{G}}_t\text{diag}(\mathbf{x}\cdot\text{id})$ as well. In particular, each partial derivative $D^{\boldsymbol{\alpha}}\hat{\mathbf{r}}^{(t)}_{\mathbf{x}_{t+1}}$ with ${\boldsymbol{\alpha}}\in\mathbb{N}^N$ is uniformly continuous on the compact $N$-dimensional interval $\prod_{i\in[N]}[-\mathbf{C}_t[i], 1+\mathbf{C}_t[i]]$, a fact we will use implicitly throughout the proof.
%
%
\subsubsection{Proof}
\label{subsec:Proof1}
The proof follows the same basic outline as in \textup{\cite{ijcai2022p676}}, but the analysis of uncertainties and nonlinearities requires substantial extensions and modifications. Here, we will focus on SSEM-UCB, and defer the changes required for SSEM-UCB-JO to Section \ref{proof:thm:2}. We start by rewriting the expected regret as
\begin{align} \nonumber
    \mathcal{R}(T) 
    &= T \mu(\mathbf{x}^{\ast}) - \sum_{t = 1}^{T} \mu(\mathbf{x}_{t}) \\
    &= \sum_{\mathbf{x}:\mu(\mathbf{x}) < \mu(\mathbf{x}^{\ast})}^{} \Delta(\mathbf{x}) \mathbb{E} [\mathcal{T}_{\mathbf{x}}(T)].
\end{align}
Based on the definition of the counters $\mathscr{T}_{i}(t)$ for the base arms $i \in [N]$, at each time $t$ that a suboptimal decision vector is selected, only one of such counters is incremented by $1$. Thus, we have \cite{Gai12:CNO}
\begin{align} \nonumber
    \mathbb{E}\left[ \sum_{\mathbf{x}:\mu(\mathbf{x}) < \mu(\mathbf{x}^{\ast})} \mathcal{T}_{\mathbf{x}}(t) \right]
    \leq T_{\mathcal{B},\mathcal{H}_{\theta},N} + \mathbb{E} \left[\sum_{i = 1}^{N} \mathscr{T}_{i}(t)\right ],
\end{align}
which implies that
\begin{align} \nonumber
    \sum_{\mathbf{x}:\mu(\mathbf{x}) < \mu(\mathbf{x}^{\ast})}^{} \mathbb{E} \left[\mathcal{T}_{\mathbf{x}}(t)\right]
    \leq T_{\mathcal{B},\mathcal{H}_{\theta},N} + \sum_{i = 1}^{N} \mathbb{E} \left[\mathscr{T}_{i}(t)\right].
\end{align}
Therefore, we observe that
\begin{align} \nonumber
    \mathcal{R}(T)
    &= \sum_{\mathbf{x}:\mu(\mathbf{x}) < \mu(\mathbf{x}^{\ast})}^{} \Delta(\mathbf{x}) \mathbb{E} [\mathcal{T}_{\mathbf{x}}(T)] \\
    &\stackrel{(\ast)}{\leq} \Delta_{\max} T_{\mathcal{B},\mathcal{H}_{\theta},N} + \Delta_{\max} \sum_{i = 1}^{N} \mathbb{E} [\mathscr{T}_{i}(T)],
\end{align}
where $(\ast)$ follows from the definition of $\Delta_{\max}$.

Let $\mathds{I}_{i}(t)$ denote the indicator function which is equal to $1$ if $\mathscr{T}_{i}(t)$ is increased by $1$ at time $t$, and is $0$ otherwise. Consequently,
\begin{align} \nonumber
    \mathscr{T}_{i}(T) = \sum_{t = T_{\mathcal{B},\mathcal{H}_{\theta},N}+1}^{T} \mathds{1}\left\{ \mathds{I}_{i}(t) = 1 \right\}.
\end{align}
If $\mathds{I}_{i}(t) = 1$, it means that a suboptimal decision vector $\mathbf{x}_{t}$ is selected at time $t$. In this case, $\mathbf{m}_{t}[i] = \min \left\{ \mathbf{m}_{t}[j] | j \in \mathcal{I}(\mathbf{x}_{t}) \right\}$.
Let \[w_{\max} = \underset{\tau\in[T]}{\max}~\underset{\lvert{\boldsymbol{\alpha}}\rvert\leq q+1}{\max}~\lvert D^{{\boldsymbol{\alpha}}}\hat{\mathbf{r}}^{(\tau-1)}_{\mathbf{x}_{\tau}}(\mathbf{a}_{\tau})\rvert,\] \[\psi_{\min} = {\scriptstyle\min~\left\{\frac{\min\{q,s\}}{w_{\max}(q+1)s^{\min\{q,s\}+1}}, \frac{1}{\Delta^{I}}\right\}}\frac{\Delta_{\min} - 4\Gamma^{(q)}_{\max}}{8}\] and \[l = \left \lceil {\frac{4 (s+1) \ln^{2}{T}}{ \psi_{\min}^{2} }} \right \rceil.\]
Then,
\begin{align} \nonumber
   &\mathscr{T}_{i}(T) 
   = \sum_{t = T_{\mathcal{B},\mathcal{H}_{\theta},N}+1}^{T} \mathds{1}\left\{ \mathds{I}_{i}(t) = 1 \right\} \\ \nonumber
   &\leq l + 
    \sum_{t = T_{\mathcal{B},\mathcal{H}_{\theta},N}+1}^{T} \mathds{1}\left\{ \mathds{I}_{i}(t) = 1 ~\&~ \mathscr{T}_{i}(t-1) \geq l \right\} \\ \nonumber
   &\leq l + \sum_{t = T_{\mathcal{B},\mathcal{H}_{\theta},N}+1}^{T} \mathds{1}\{ I^{\max}_{t}(\mathbf{x}^{\ast}) + E_t(\mathbf{x}^{\ast}) \\\nonumber
   &\hspace{15mm}\leq I^{\max}_{t}(\mathbf{x}_{t}) + E_t(\mathbf{x}_t) ~\&~ \mathscr{T}_{i}(t-1) \geq l\} \\ \nonumber
   &= l + \sum_{t = T_{\mathcal{B},\mathcal{H}_{\theta},N}}^{T} \mathds{1}\{ I^{\max}_{t+1}(\mathbf{x}^{\ast}) + E_{t+1}(\mathbf{x}^{\ast}) \\\nonumber
   &\hspace{15mm}\leq I^{\max}_{t+1}(\mathbf{x}_{t+1}) + E_{t+1}(\mathbf{x}_{t+1}) ~\&~ \mathscr{T}_{i}(t) \geq l\}.
\end{align}
Based on the definition of $\mathscr{T}_{i}(t)$, we have $\mathscr{T}_{i}(t) \leq \mathbf{m}_{t}[i]$, $\forall i \in [N]$. Hence, when $\mathscr{T}_{i}(t) \geq l$, we know that \cite{Gai12:CNO}
\begin{align} \nonumber
    l \leq \mathscr{T}_{i}(t) \leq \mathbf{m}_{t}[j], ~~~~ \forall j \in \mathcal{I}(\mathbf{x}_{t+1}).
\end{align}
Let $\mathbf{v}_{t+1}^{\top} = {\bf 1}^{\top} (\mathbf{I} - \hat{\mathbf{F}}_{t})^{-1} \hat{\mathbf{G}}_{t}$. We order the elements in sets $\mathcal{I}(\mathbf{x}^{\ast})$ and $\mathcal{I}(\mathbf{x}_{t+1})$ arbitrarily. In the following, our results are independent of the way we order these sets. Let $v_{k}$, $k = 1, \dots, |\mathcal{I}(\mathbf{x}^{\ast})| \leq s$, represent the $k$th element in $\mathcal{I}(\mathbf{x}^{\ast})$ and $u_{k}$, $k = 1, \dots, |\mathcal{I}(\mathbf{x}_{t+1})| \leq s$, represent the $k$th element in $\mathcal{I}(\mathbf{x}_{t+1})$. Explicitly, we have
{\allowdisplaybreaks
\begin{align}\nonumber
   &\mathscr{T}_{i}(T)
   \leq l + \sum_{t = T_{\mathcal{B},\mathcal{H}_{\theta},N}}^{T} \mathds{1} {\Bigg\{} \min_{0 < \mathbf{m}[v_{1}], \dots, \mathbf{m}[v_{|\mathcal{I}(\mathbf{x}^{\ast})|}] \leq t} \\ \nonumber
   &\hspace{3mm}\sum_{\substack{\lvert{\boldsymbol{\alpha}}\rvert\leq q\\\mathbf{x}^{\ast}\odot{\boldsymbol{\alpha}}={\boldsymbol{\alpha}}}} \frac{D^{{\boldsymbol{\alpha}}}\mathbf{v}_{t+1}^{\top}\text{diag}(\mathbf{x}^{\ast}\cdot\text{id})\left(\mathbf{a}_{t+1})\right)}{{\boldsymbol{\alpha}}!} \\ \nonumber
   &\hspace{3mm}\prod_{j=1}^{\lvert\mathcal{I}(\mathbf{x}^{\ast})\rvert}(\hat{\boldsymbol{\phi}}^{({\boldsymbol{\alpha}}[v_j])}_t[v_j] + \mathcal{W}^{({\boldsymbol{\alpha}}[j])_{\max}}_t(\mathbf{x}^{\ast})[v_j]) + E_{t+1}(\mathbf{x}^{\ast})\\ \nonumber
   &\hspace{3mm}\leq E_{t+1}(\mathbf{x}_{t+1}) + \max_{l \leq \mathbf{m}[u_{1}], \dots,
   \mathbf{m}[u_{|\mathcal{I}(\mathbf{x}_{t+1})|}] \leq t} \sum_{\substack{\lvert{\boldsymbol{\alpha}}\rvert\leq q\\\mathbf{x}_{t+1}\odot{\boldsymbol{\alpha}}={\boldsymbol{\alpha}}}}\\ \nonumber
   &\hspace{3mm}\frac{D^{{\boldsymbol{\alpha}}}\mathbf{v}_{t+1}^{\top}\text{diag}(\mathbf{x}_{t+1}\cdot\text{id})\left(\mathbf{a}_{t+1}\right)}{{\boldsymbol{\alpha}}!} \\ \nonumber
   &\hspace{3mm}\prod_{j=1}^{\lvert\mathcal{I}(\mathbf{x}_{t+1})\rvert}(\hat{\boldsymbol{\phi}}^{({\boldsymbol{\alpha}}[u_j])}_t[u_j] + \mathcal{W}^{({\boldsymbol{\alpha}}[j])\max}_t(\mathbf{x}_{t+1})[u_j]) {\Bigg \}} \\ \nonumber
   &\leq l + \sum_{t=T_{\mathcal{B},\mathcal{H}_{\theta},N}}^{\infty} \sum_{m_{v_{1}}=1}^{t} \dots \hspace{-2mm}\sum_{m_{v_{|\mathcal{I}(\mathbf{x}^{\ast})|}}=1}^{t} \sum_{m_{u_{1}}=l}^{t} \dots \hspace{-3mm}\sum_{m_{u_{|\mathcal{I}(\mathbf{x}_{t+1})|}}=l}^{t} \\ \nonumber
   &\hspace{1mm}\mathds{1} {\Bigg \{} \sum_{\substack{\lvert{\boldsymbol{\alpha}}\rvert\leq q\\\mathbf{x}^{\ast}\odot{\boldsymbol{\alpha}}={\boldsymbol{\alpha}}}}\frac{D^{{\boldsymbol{\alpha}}}\mathbf{v}_{t+1}^{\top}\left(\sum\limits_{j=1}^{|\mathcal{I}(\mathbf{x}^{\ast})|} (\mathbf{a}_{t+1}[v_{j}])\mathbf{I}[v_{j}]^\top\right)}{{\boldsymbol{\alpha}}!} \\ \nonumber
   &\hspace{1mm}\prod_{j=1}^{\lvert\mathcal{I}(\mathbf{x}^{\ast})\rvert}(\hat{\boldsymbol{\phi}}^{({\boldsymbol{\alpha}}[v_j])}_t[v_j] + \mathcal{W}^{({\boldsymbol{\alpha}}[j])\max}_t(\mathbf{x}^{\ast})[v_j]) + E_{t+1}(\mathbf{x}^{\ast}) \\ \nonumber
   &\hspace{3mm}\leq\sum_{\substack{\lvert{\boldsymbol{\alpha}}\rvert\leq q\\\mathbf{x}_{t+1}\odot{\boldsymbol{\alpha}}={\boldsymbol{\alpha}}}}\frac{D^{{\boldsymbol{\alpha}}}\mathbf{v}_{t+1}^{\top}\left(\sum\limits_{j=1}^{|\mathcal{I}(\mathbf{x}_{t+1})|} (\mathbf{a}_{t+1}[u_{j}])\mathbf{I}[u_{j}]^\top\right)}{{\boldsymbol{\alpha}}!} \\\nonumber
   &\hspace{1mm}\prod_{j=1}^{\lvert\mathcal{I}(\mathbf{x}_{t+1})\rvert}(\hat{\boldsymbol{\phi}}^{({\boldsymbol{\alpha}}[u_j])}_t[u_j] + \mathcal{W}^{({\boldsymbol{\alpha}}[j])\max}_t(\mathbf{x}_{t+1})[u_j]) \\\label{eq:event}
   &\hspace{5mm}+ E_{t+1}(\mathbf{x}_{t+1})
   {\Bigg \}}. 
\end{align}}
Given $m_{v_1},\dots,m_{v_{\lvert\mathcal{I}(\mathbf{x}^{\ast})\rvert}}, m_{u_1},\dots,m_{u_{\lvert\mathcal{I}(\mathbf{x}_{t+1})\rvert}}$, we define event $\mathcal{P}$ as
\begin{align} \nonumber
    I^{\max}_{t+1}(\mathbf{x}^{\ast}) + E_{t+1}(\mathbf{x}^{\ast}) \leq I^{\max}_{t+1}(\mathbf{x}_{t+1}) + E_{t+1}(\mathbf{x}_{t+1}).
\end{align}
Event $\mathcal{P}$ would necessitate at least one of the following events to apply:
\begin{align}
    \label{eq:part1}
    &I^{\max}_{t+1}(\mathbf{x}^{\ast}) +  \Delta_{\min}/4 + \Gamma^{(q)}_{\max} < \mu(\mathbf{x}^{\ast}), \\ \label{eq:part2}
    &I^{\min}_{t+1}(\mathbf{x}_{t+1}) -  \Delta_{\min}/4 - \Gamma^{(q)}_{\max} > \mu(\mathbf{x}_{t+1}), \\ \label{eq:part3}
    &\mu(\mathbf{x}^{\ast}) - \mu(\mathbf{x}_{t+1}) \leq I^{\max}_{t+1}(\mathbf{x}_{t+1}) - I^{\min}_{t+1}(\mathbf{x}_{t+1}) \\\nonumber
    &\hspace{10mm} + E_{t+1}(\mathbf{x}_{t+1}) + \Delta_{\min}/2 + 2\Gamma^{(q)}_{\max}.
\end{align}
\\
First, we consider (\ref{eq:part1}).
We set
\begin{align}\label{eq:J_t}
J_t(\mathbf{x}) = \sum_{\lvert{\boldsymbol{\alpha}}\rvert\leq q}\frac{D^{{\boldsymbol{\alpha}}}\hat{\mathbf{r}}^{(t-1)}_\mathbf{x}(\mathbf{a}_{t+1})}{{\boldsymbol{\alpha}}!}\prod_{j\in\mathcal{I}(\mathbf{x})}\left(\boldsymbol{\phi}^{({\boldsymbol{\alpha}}[j])}[j]\right)
\end{align}
and investigate the event $\mathcal{V}$
\begin{align}\nonumber
	I^{\max}_{t+1}(\mathbf{x}^{\ast}) &+ J_{t+1}(\mathbf{x}^{\ast}) +  \Delta_{\min}/4 + \Gamma^{(q)}_{\max} \\\nonumber
	&< J_{t+1}(\mathbf{x}^{\ast}) + \mu(\mathbf{x}^{\ast}).
\end{align}
If $\mathcal{V}$ is true, then at least one of the following must hold.
\begin{align}\nonumber
    &\underbrace{I^{\max}_{t+1}(\mathbf{x}^{\ast}) < J_{t+1}(\mathbf{x}^{\ast})}_{\mathcal{I}}, \\\nonumber
    &\underbrace{J_{t+1}(\mathbf{x}^{\ast}) + \Delta_{\min}/4 + \Gamma^{(q)}_{\max} < \mu(\mathbf{x}^{\ast})}_{\mathcal{II}}.
\end{align}
Therefore, we have
\begin{align}\label{eq:eventE}
    \mathbb{P} \left[ \mathcal{V} \right] \leq \mathbb{P} \left[ \mathcal{I} \right] + \mathbb{P} \left[ \mathcal{II} \right].
\end{align}
For event $\mathcal{I}$ to be true, by selection of $\mathcal{W}^{(m)\max}_t(\mathbf{x}^{\ast})$, there would have to be $k\in[\lvert\mathcal{I}(\mathbf{x}^{\ast})\rvert]$ and $m\in[q]$ with
\begin{align}\label{eq:VarEst}
    2^{2\text{ mod }m}\mathbf{C}_{t}[v_{k}] < \left\lvert\boldsymbol{\phi}^{(m)}[v_{k}] - \hat{\boldsymbol{\phi}}^{(m)}_{t}[v_{k}]\right\rvert.
\end{align}
Considering each case, we have
\begin{align} \nonumber
    &\mathbb{P}{\Big(} 2^{2\text{ mod }m}\mathbf{C}_{t}[v_{k}] < \left\lvert\boldsymbol{\phi}^{(m)}[v_{k}] - \hat{\boldsymbol{\phi}}^{(m)}_{t}[v_{k}]\right\rvert {\Big)} \\ \nonumber
    &\hspace{5mm}= \mathbb{P}{\Big(} 2^{2\text{ mod }m}\mathbf{m}_{t}[v_{k}] \mathbf{C}_{t}[v_{k}] \\ \nonumber
    &\hspace{15mm}< \mathbf{m}_{t}[v_{k}] \left\lvert\boldsymbol{\phi}^{(m)}[v_{k}] - \hat{\boldsymbol{\phi}}^{(m)}_{t}[v_{k}]\right\rvert{\Big)} \\ \nonumber
    &\hspace{5mm}\stackrel{(a)}{\leq} 2 e^{-(2/\mathbf{m}_{t}[v_{k}]) \mathbf{m}_{t}[v_{k}]^{2} \mathbf{C}_{t}[v_{k}]^{2}} \\ \nonumber
    &\hspace{5mm}\stackrel{(b)}{=} 2 e^{-2 (s+1) \ln t} \\
    &\hspace{5mm}= 2 t^{-2 (s+1)},
\end{align}
where $(a)$ follows from Lemma \ref{lem:Hoeffding}, and $(b)$ from the definition of $\mathbf{C}_{t}$. Since this bound applies to each case's probability, for Event $\mathcal{I}$, we conclude that
\begin{align}\label{eq:PI}
    \mathbb{P}\left[\mathcal{I}\right] \leq q|\mathcal{I}(\mathbf{x}^{\ast})| 2 t^{-2 (s+1)} \leq 2 q s t^{-2 (s+1)}.
\end{align}
Now, we investigate Event $\mathcal{II}$:
\begin{align}\nonumber
    &\mathbb{P}\left(\left\lvert \mu(\mathbf{x}^{\ast}) - J_{t+1}(\mathbf{x}^{\ast})\right\rvert >  \Delta_{\min}/4 + \Gamma^{(q)}_{\max}\right) \\ \nonumber
    &\hspace{5mm}\stackrel{(a)}{\leq} \mathbb{P}\left(\left\lvert \mathbb{E}[\mathbf{r}_{\mathbf{x}^{\ast}}(\mathbf{b}) - \hat{\mathbf{r}}^{(t)}_{\mathbf{x}^{\ast}}(\mathbf{b})]\right\rvert >  \Delta_{\min}/4\right) \\ \nonumber
    &\hspace{5mm}= \mathbb{P}\left(\left( \mathbb{E}[\mathbf{r}_{\mathbf{x}^{\ast}}(\mathbf{b}) - \hat{\mathbf{r}}^{(t)}_{\mathbf{x}^{\ast}}(\mathbf{b})]\right)^2 >  \Delta_{\min}^2/16\right) \\ \label{eq:ev2opt}
    &\hspace{5mm}\stackrel{(b)}{\leq} \mathbb{P}\left(\mathbb{E}[(\mathbf{r}_{\mathbf{x}^{\ast}}(\mathbf{b}) - \hat{\mathbf{r}}^{(t)}_{\mathbf{x}^{\ast}}(\mathbf{b}))^2] >  \Delta_{\min}^2/16\right),
\end{align}
where $(a)$ results from the definition of $\Gamma^{(q)}_{\max}$, and $(b)$ is implied by Jensen's inequality. Next, we will derive an ordinal bound to this probability from our optimization guarantees.

We interpret the exogenous vectors observed at time steps before $t$ when base action $o\in[N]$ was part of the super arm as sampled from the empirical distribution $\mathcal{B}_t^{(o)}$, defined by $\mathbb{P}(\mathbf{Z}=\mathbf{z} | \hat{t}\sim\text{Uni}(\{\overline{t}\in[t]|o\in\mathcal{I}(\mathbf{x_{\overline{t}}})\}, \mathbf{z}\sim\mathcal{B}^{\mathbf{x}_{\hat{t}}})$.
For $i\in[N]$ we write
\begin{align}\nonumber
    \hat{\mathbf{F}}_{t}^{(o)}[i] &= \underset{h\in\mathcal{H}_{\theta}^{(i)}}{\arg\min}~\sum_{\tau=1}^{t}(h(\mathbf{y}(\mathbf{z}_{\tau}))-\mathbf{F}(\mathbf{y}(\mathbf{z}_{\tau}))[i])^2 \\\nonumber
    &=\underset{h\in\mathcal{H}_{\theta}^{(i)}}{\arg\min}~\sum_{\tau=1}^{t}(h(\mathbf{y}(\mathbf{z}_{\tau}))-\mathbf{y}(\mathbf{z}_{\tau})[i]+\mathbf{z}_{\tau}[i])^2
\end{align}
for the solution to the minimization problem with norm constraint $\theta$ (see Algorithm \ref{Alg:SSEM-UCB}). A Rademacher complexity analysis of our kernel optimization (as described in Prop. 2 of \textup{\cite{Bertsimas21:DDO}}) results in an accuracy bound of
\begin{align}\label{eq:rad}
   \varrho_t &\geq \mathbb{P}(\mathbb{E}_{\mathcal{B}_t^{(o)}} [(\hat{\mathbf{F}}_{t}^{(o)}(\mathbf{y}(\mathbf{b}))[i]-\mathbf{F}(\mathbf{y}(\mathbf{b}))[i])^2] \\\nonumber
    &- \mathbb{E}_{\mathcal{B}_t^{(o)}} [(\mathbf{F}(\mathbf{y}(\mathbf{b}))[i]-\mathbf{F}_{\mathcal{B}_{t}^{(o)}}(\mathbf{y}(\mathbf{b}))[i])^2]> H_{t}^{(o)})
\end{align}
with $\mathbf{F}_{\mathcal{B}_t^{(o)}}[i]=\underset{h\in\mathcal{H}^{(i)}_{\theta}}{\arg\min}~d_{\mathcal{B}_t^{(o)}}(h\mathbf{y}, \mathbf{F}\mathbf{y}[i])$, $\varrho_t=1-\sqrt[N^6]{1-t^{-2(s+1)}}$, \[H_{t}^{(o)} = M_1\frac{1}{\sqrt{\mathbf{m}_t[o]}} + M_2\sqrt{\frac{\ln (2/\varrho_t)}{2\mathbf{m}_t[o]}},\] where $M_1 = 4(\kappa\sqrt{\theta}+\underset{\mathbf{b}\in[0,1]^N}{\max}\mathbf{r}(\mathbf{b}))\kappa\sqrt{\theta}$, $M_2 = 2(\kappa\sqrt{\theta}+\underset{\mathbf{b}\in[0,1]^N}{\max}\mathbf{r}(\mathbf{b}))^2$, $\kappa=\sqrt{\underset{b\in[0,1]}{\sup} K(b,b)}$ and $K$ is the positive semi-definite kernel associated with the RKHS $\mathcal{H}^{(i)}$. Since $K$ is continuous, a finite supremum exists over the closed interval domain of functions in $\mathcal{H}^{(i)}$. Even though the objective gap \[\varphi_{\mathcal{H}_{\theta}} = \underset{i\in[N]}{\max}\underset{\hat{\mathbf{F}}[i]\in\mathcal{H}_{\theta}^{(i)}}{\min}\underset{\mathbf{x}\in\mathcal{X}}{\max} ~\sqrt{\mathbb{E}[(\mathbf{F}-\hat{\mathbf{F}})(\mathbf{y}_{\mathbf{x}}(\mathbf{b}))[i]^2]}\] is zero as by the assumption of Theorem \ref{thm:1}, we treat it as a variable to support the generalizations in Theorem \ref{thm:4}.

For any $\mathbf{x}$ we can represent $\mathcal{B}^{\mathbf{x}}$ as the distribution induced by $\mathcal{B}$ over $[0,1]^{\mathbf{x}}=\bigtimes\limits_{i\in[N]}[0,1]^{\lvert\mathcal{I}(\mathbf{x})\cap\{i\}\rvert}$, which is nonzero almost everywhere. In particular, $\tilde{\mathcal{B}_o}\coloneqq\mathcal{B}^{\mathbf{I}[o]^{\top}}$ assigns $0$ in every dimension other than the $o$th.
Taking into account that $d_{\mathcal{B}_t^{(o)}}(\mathbf{F}\mathbf{y}[i],\mathbf{F}_{\mathcal{B}_t^{(o)}}\mathbf{y}[i])\leq d_{\mathcal{B}_t^{(o)}}(\mathbf{F}\mathbf{y}[i],\mathbf{F}_{\mathcal{H}_{\theta}}\mathbf{y}[i])\leq\varphi_{\mathcal{H}_{\theta}}$ for a realization $\mathbf{F}_{\mathcal{H}_{\theta}}$ of $\varphi_{\mathcal{H}_{\theta}}$, we deduce that
\begin{align}\nonumber
   \varrho_t &\stackrel{(\ref{eq:rad})}{\geq}\mathbb{P}(\underset{\mathbf{x}\in\mathcal{X}:o\in\mathcal{I}(\mathbf{x})}{\min}\mathbb{E}_{\mathcal{B}^{\mathbf{x}}} [(\hat{\mathbf{F}}_{t}^{(o)}(\mathbf{y}(\mathbf{b}))[i]-\mathbf{F}(\mathbf{y}(\mathbf{b}))[i])^2] \\\nonumber
    &\hspace{11mm}> H_t^{(o)}+\varphi_{\mathcal{H}_{\theta}}^2)\\\nonumber
    &\geq \mathbb{P}(\mathbb{E}_{\tilde{\mathcal{B}}_{o}} [(\hat{\mathbf{F}}_{t}^{(o)}(\mathbf{y}(\mathbf{b}))[i]-\mathbf{F}(\mathbf{y}(\mathbf{b}))[i])^2] \\\label{eq:EF}
    &\hspace{11mm}> \frac{u_{\mathcal{B}}}{l_{\mathcal{B}}^s}(H_t^{(o)}+\varphi_{\mathcal{H}_{\theta}}^2)),
\end{align}
where $u_{\mathcal{B}}, l_{\mathcal{B}} > 0$ are an upper respective lower bound on the factors $(\sum_{i=1}^{l}\sqrt{\mathbb{E}_{\tilde{\mathcal{B}}_{o}}[e_i^2]})^2$, $l_{\mathbf{y}(\mathcal{B}^\mathbf{x})}$ in Lemma \ref{lem:dish} for all distributions $\mathbf{y}(\mathcal{B}^\mathbf{x})$, $\mathbf{x}\in\mathcal{X}$, $o\in[N]$. Had we not fixed the exogenous operator $\mathbf{G}=\mathbf{I}$, we would jointly optimize for its prediction alongside $\mathbf{F}$ and directly gain a prediction certainty for $\mathbf{G}[o,o]$ regarding $\tilde{\mathcal{B}}_{o}$, which we would utilize analogously in subsequent steps.

Since $F$ is not cyclic, the SEM describes the recursive formula of backward substitution:
\begin{equation}\label{eq:bs}
    (\mathbf{I}-\mathbf{F})^{-1}(\mathbf{b})[i] = \mathbf{b}[i]+\sum_{j=i+1}^{N}f_{ij}\left((\mathbf{I}-\mathbf{F})^{-1}(\mathbf{b})[j]\right),
\end{equation}
where $i=N,\dots,1$, $\mathbf{b}\sim\mathcal{B}$.
As such, we can express the expectation for $\hat{\mathbf{y}}^{(o)}_t = (\mathbf{I}-\hat{\mathbf{F}}^{(o)}_t)^{-1}$ with
\begin{align}\label{eq:Ey}
	&d_{\tilde{\mathcal{B}}_{o}}(\hat{\mathbf{y}}^{(o)}_t[i],\mathbf{y}[i]) \\\nonumber
	&= d_{\tilde{\mathcal{B}}_{o}}(\hat{\mathbf{F}}^{(o)}_t\hat{\mathbf{y}}^{(o)}_{t}[i],\mathbf{F}\mathbf{y}[i]) \\\nonumber
	&\leq d_{\tilde{\mathcal{B}}_{o}}(\hat{\mathbf{F}}^{(o)}_t\mathbf{y}[i],\mathbf{F}\mathbf{y}[i]) + d_{\tilde{\mathcal{B}}_{o}}(\hat{\mathbf{F}}^{(o)}_t\hat{\mathbf{y}}^{(o)}_{t}[i],\hat{\mathbf{F}}^{(o)}_{t}\mathbf{y}[i]) \\\nonumber
	&\stackrel{(a)}{\leq} d_{\tilde{\mathcal{B}}_{o}}(\hat{\mathbf{F}}^{(o)}_{t}\mathbf{y}[i], \mathbf{F}\mathbf{y}[i]) + \eta\sum\limits_{j=i+1}^{o-1}d_{\tilde{\mathcal{B}}_{o}}(\hat{\mathbf{y}}^{(o)}_t[j],\mathbf{y}[j]) \\\nonumber
	&\stackrel{(b)}{=} d_{\tilde{\mathcal{B}}_{o}}(\hat{\mathbf{F}}^{(o)}_{t}\mathbf{y}[i], \mathbf{F}\mathbf{y}[i]) + \eta\sum\limits_{k=0}^{o-2-i}(\eta+1)^k \\\nonumber
	&\hspace{5mm}d_{\tilde{\mathcal{B}}_{o}}(\hat{\mathbf{F}}^{(o)}_{t}\mathbf{y}[i+1+k], \mathbf{F}\mathbf{y}[i+1+k]),
\end{align}
where $(b)$ transits to the explicit form of the recursive formula and $(a)$ utilizes the continuity of $\hat{\mathbf{F}}^{(o)}_t[i]$. Here, $\eta$ is an upper bound on all the $\tilde{\mathcal{B}}_{o}$, $o\in[N]$, associated distances of the gradient for all the bounded functions in $\mathcal{H}_{\theta}$, which exists since derivation constitutes a linear map.
This gives rise to
\begin{align}\nonumber
	&\mathbb{P}(\mathbb{E}_{\tilde{\mathcal{B}}_{o}}[(\hat{\mathbf{y}}^{(o)}_t(\mathbf{b})[i]-\mathbf{y}(\mathbf{b})[i])^2] \\\nonumber
	&\hspace{5mm}\stackrel{(\ref{eq:EF})}{>} \frac{u_{\mathcal{B}}}{l_{\mathcal{B}}^s}(H_t^{(o)}+\varphi_{\mathcal{H}_{\theta}}^2)(1+\eta\sum\limits_{k=0}^{o-2-i}(\eta+1)^k))^2 \\\label{eq:PEy}
	&\hspace{2mm}\leq1-(1-\varrho_t)^{o-1}.
\end{align}
The prediction accuracies over the separated $\tilde{\mathcal{B}}_{o}$'s will provide an accuracy bound for each component function of $\hat{\mathbf{F}}_{t}=(\hat{f}_{ij}^{(t)})_{i,j\in[N]}$, as defined in Algorithm \ref{Alg:SSEM-UCB}, by consulting (\ref{eq:bs}) again. To this end, batches $(\hat{f}_{o-k,o}^{(t)})_{o\in\{k+1,\dots,N\}}$ can be assessed in chronological order $k=1,\dots,N-1$ as
\begin{align}\label{eq:Ef}
    d_{\mathcal{B}_o}(&\hat{f}_{o-k,o}^{(t)},f^{(\tilde{\mathcal{B}}_{o})}_{o-k,o}) \\\nonumber
    &\leq d_{\tilde{\mathcal{B}}_{o}}\left(\hat{\mathbf{y}}_t^{(o)}[o-k],\mathbf{y}_{\tilde{\mathcal{B}}_{o}}[o-k]\right) \\\nonumber
    &+\sum\limits_{j=o-k+1}^{o-1}d_{\tilde{\mathcal{B}}_{o}}\left(\hat{f}_{o-k,j}^{(t)}(\hat{\mathbf{y}}_t^{(o)}[j]),f^{(\tilde{\mathcal{B}}_{o})}_{o-k,j}(\mathbf{y}_{\tilde{\mathcal{B}}_{o}}[j])\right) \\\nonumber
    &\leq d_{\tilde{\mathcal{B}}_{o}}\left(\hat{\mathbf{y}}_t^{(o)}[o-k],\mathbf{y}_{\tilde{\mathcal{B}}_{o}}[o-k]\right) \\\nonumber
    &+\sum\limits_{j=o-k+1}^{o-1}d_{\tilde{\mathcal{B}}_{o}}\left(\hat{f}_{o-k,j}^{(t)}(\mathbf{y}_{\tilde{\mathcal{B}}_{o}}[j]),f^{(\tilde{\mathcal{B}}_{o})}_{o-k,j}(\mathbf{y}_{\tilde{\mathcal{B}}_{o}}[j])\right) \\\nonumber
    &+\sum\limits_{j=o-k+1}^{o-1}d_{\tilde{\mathcal{B}}_{o}}\left(\hat{f}_{o-k,j}^{(t)}(\hat{\mathbf{y}}_t^{(o)}[j]),\hat{f}_{o-k,j}^{(t)}(\mathbf{y}_{\tilde{\mathcal{B}}_{o}}[j])\right) \\\nonumber
    &\leq d_{\tilde{\mathcal{B}}_{o}}\left(\hat{\mathbf{y}}_t^{(o)}[o-k],\mathbf{y}_{\tilde{\mathcal{B}}_{o}}[o-k]\right) \\\nonumber
    &+\sum\limits_{j=o-k+1}^{o-1}\mathcal{K}(\mathbf{I}[o]^{\top},\mathbf{y}_{\tilde{\mathcal{B}}_{o}})d_{\mathcal{B}_j}\left(\hat{f}_{o-k,j}^{(t)},f^{(\tilde{\mathcal{B}}_{o})}_{o-k,j}\right) \\\nonumber
    &+\sum\limits_{j=o-k+1}^{o-1}g_1(\eta)d_{\tilde{\mathcal{B}}_{o}}\left(\hat{\mathbf{y}}_t^{(o)}[j],\mathbf{y}_{\tilde{\mathcal{B}}_{o}}[j]\right).
\end{align}
Here, $\mathbf{y}_{\tilde{\mathcal{B}}_{o}}=(\mathbf{I}-\mathbf{F}_{\tilde{\mathcal{B}}_{o}})^{-1}$, $\mathbf{F}_{\tilde{\mathcal{B}}_{o}}[i]=(f_{i,j}^{(\tilde{\mathcal{B}}_{o})})_{j\in[N]}=\underset{h\in\mathcal{H}^{(i)}_{\theta}}{\arg\min}~d_{\tilde{\mathcal{B}}_{o}}(h\mathbf{y}, \mathbf{F}\mathbf{y}[i])$, $\mathcal{K}(\mathbf{x},y)=\underset{j\in[N]}{\max}\sum_{i=1}^{\dim(\mathcal{F}^{(j)})}\sqrt{\mathbb{E}_{\mathcal{B}^{\mathbf{x}}}[e_i(y(\mathbf{b}))^2]}/l_{\mathcal{B}}$, and $(e_1,\dots,e_{\dim(\mathcal{F}^{(j)})})$ is a basis of the space $\mathcal{F}^{(j)}$ containing the algorithm's possible values for the entries of $\hat{\mathbf{F}}_t$ that is orthonormal over $[0,1]$ in accordance to Lemma \ref{lem:dish}. (For SSEM-UCB-JO, $\mathcal{F}^{(j)}=\mathcal{H}^{(j)}$.) Further, $g_1(\eta)$ with $g_1:\mathbb{R}_{\geq0}\rightarrow\mathbb{R}_{\geq0}$ exists as a gradient bound on each component function. A numerical value of this bound can be determined by replicating the transformation in (\ref{eq:Ey}) for $\lvert\hat{\mathbf{y}}^{(o)}_t(b\mathbf{I}[o]^{\top})[j]-\hat{\mathbf{y}}^{(o)}_t(c\mathbf{I}[o]^{\top}[j])\rvert~\forall j\in[N]$, yielding $g_1(\eta)=(1+\eta(1+\eta)^N)^N$). (For SSEM-UCB-JO, $g_1(\eta)=\eta$ would suffice as each component function lies in one of the RKHSs.)

Replicating (\ref{eq:Ey}) also reveals
\begin{align}\nonumber
    d_{\tilde{\mathcal{B}}_{o}}&\left(\hat{\mathbf{y}}_t^{(o)}[j],\mathbf{y}_{\tilde{\mathcal{B}}_{o}}[j]\right) \\\nonumber
    &\leq d_{\tilde{\mathcal{B}}_{o}}\left(\hat{\mathbf{y}}_t^{(o)}[j],\mathbf{y}[j]\right) + d_{\tilde{\mathcal{B}}_{o}}\left(\mathbf{y}_{\tilde{\mathcal{B}}_{o}}[j],\mathbf{y}[j]\right)\\\nonumber
    &\leq d_{\tilde{\mathcal{B}}_{o}}\left(\hat{\mathbf{y}}_t^{(o)}[j],\mathbf{y}[j]\right) + g_2(\varphi_{\mathcal{H}_{\theta}})
\end{align}
with $g_2(\varphi_{\mathcal{H}_{\theta}})=\varphi_{\mathcal{H}_{\theta}}(1+\eta\sum\limits_{k=0}^{o-2-j}(\eta+1)^k)$. Resolving the recursive form of (\ref{eq:Ef}), we then find that
\begin{align}\nonumber
	\mathbb{P}\Bigg(&\mathbb{E}_{\mathcal{B}_o}\left[\left(\hat{f}_{o-k,o}^{(t)}(b)-f^{(\tilde{\mathcal{B}}_{o})}_{o-k,o}(b)\right)^2\right] \\\nonumber
	&\hspace{2mm}> (1+\mathcal{K}(\mathbf{I}[o]^{\top},\mathbf{y}_{\tilde{\mathcal{B}}_{o}})\sum\limits_{j=0}^{k-2}(\mathcal{K}(\mathbf{I}[o]^{\top},\mathbf{y}_{\tilde{\mathcal{B}}_{o}})+1)^j)^2 \\\nonumber
	&\hspace{6mm}(1+(k-1)g_1(\eta))^2(\sqrt{L_t^{(o)}}+g_2(\varphi_{\mathcal{H}_{\theta}}))^2\Bigg) \\\label{eq:Pf}
	&\stackrel{(\ref{eq:PEy})}{\leq} 1-(1-\varrho_t)^{(o-1)(k-1)k},
\end{align}
abbreviating $L_t^{(o)}=\frac{u_{\mathcal{B}}}{l_{\mathcal{B}}^s}(H_t^{(o)}+\varphi_{\mathcal{H}_{\theta}}^2)(1+\eta\sum\limits_{k=0}^{o-2-i}(\eta+1)^k))^2$ and $M_t^{(o,k)}$ for the lower side of the inequality in the above probability.

We can now show that
\begin{align}\nonumber
	d_{\mathcal{B}^{\mathbf{x^{\ast}}}}&(\hat{\mathbf{y}}_t[i],\mathbf{y}[i]) \\\nonumber
	&\leq \sum_{j=i+1}^{N}d_{\mathcal{B}^{\mathbf{x}^{\ast}}}(\hat{f}_{ij}^{(t)}(\hat{\mathbf{y}}_t[j]),f_{ij}(\mathbf{y}[j])) \\\nonumber
	&\leq \sum_{j=i+1}^{N}\Big(d_{\mathcal{B}^{\mathbf{x}^{\ast}}}(\hat{f}_{ij}^{(t)}(\mathbf{y}[j]),f^{(\mathcal{B}^{\mathbf{I}[j]^{\top}})}_{ij}(\mathbf{y}[j]))  \\\nonumber
	&\hspace{10mm}+ d_{\mathcal{B}^{\mathbf{x}^{\ast}}}(f_{ij}(\mathbf{y}[j]),f^{(\mathcal{B}^{\mathbf{I}[j]^{\top}})}_{ij}(\mathbf{y}[j]))\Big) \\\nonumber
	&\hspace{5mm}+ \sum_{j=i+1}^{N}d_{\mathcal{B}^{\mathbf{x}^{\ast}}}(\hat{f}_{ij}^{(t)}(\hat{\mathbf{y}}_t[j]),\hat{f}_{ij}^{(t)}(\mathbf{y}[j])) \\\nonumber
	&\leq \sum_{j=i+1}^{N}\mathcal{K}(\mathbf{x}^{\ast},\mathbf{y})\left(d_{\mathcal{B}_j}(\hat{f}_{ij}^{(t)},f_{ij}) + g_3(\varphi_{\mathcal{H}_{\theta}})\right)  \\\nonumber
	&\hspace{5mm}+ \sum_{j=i+1}^{N}g_1(\eta)d_{\mathcal{B}^{\mathbf{x}^{\ast}}}(\hat{\mathbf{y}}_t[j],\mathbf{y}[j]).
\end{align}
The numerical value of $g_3(\varphi_{\mathcal{H}_{\theta}})$ with zero-true $g_3:\mathbb{R}_{\geq0}\rightarrow\mathbb{R}_{\geq0}$ can be stated along the lines of (\ref{eq:Ef}) and similar to $M_t^{(o,k)}$, which we forego here.

Finally our analysis yields
\begin{align}\nonumber
	\mathbb{P}\Big(\mathbb{E}_{\mathcal{B}^{\mathbf{x}^{\ast}}}&[(\hat{\mathbf{r}}^{(t)}(\mathbf{b})-\mathbf{r}(\mathbf{b}))^2] >G_t^2\Big) \\\nonumber
	&\stackrel{(\ref{eq:Pf})}{\leq} 1-(1-\varrho_t)^{N^6} = t^{-2(s+1)}
\end{align}
for
\begin{align}\nonumber
	G_t = &(1+g_1(\eta)\sum\limits_{k=0}^{N-1}(g_1(\eta)+1)^k) \sqrt{N}^3 \\\label{def:opt_gua}
	&\mathcal{K}(\mathbf{x}^{\ast},\mathbf{y})\left(\sqrt{\underset{o,k\in[N]}{\max}M_t^{(o,k)}} + g_3(\varphi_{\mathcal{H}_{\theta}})\right).
\end{align}
From our selection criterion we know that
\begin{align}\nonumber
    0 &\leq I^{\max}_{t+1}(\mathbf{x}_{t+1}) + \Delta_{t+1}^{I}\sqrt{\frac{16(s+1)\ln^2(t)}{\underset{i\in\mathcal{I}(\mathbf{x}_{t+1})}{\min}\mathbf{m}_t[i]}}\\\nonumber
    &\hspace{-1mm}- I^{\max}_{t+1}(\mathbf{I}[\underset{i\in[N]}{\arg \min}~\mathbf{m}_t[i]]^{\top}) - \Delta_{t+1}^{I}\sqrt{\frac{16(s+1)\ln^2(t)}{\underset{i\in[N]}{\min}\mathbf{m}_t[i]}}\\\nonumber
    &\hspace{-1mm}\leq \Delta^{I}_{t+1}\left(1 + 4\sqrt{s+1}\ln(t)\left(\frac{1}{\sqrt{l}} - \frac{1}{\sqrt{\underset{i\in[N]}{\min}\mathbf{m}_t[i]}}\right)\right),
\end{align}
and, as ties are broken randomly, w.l.o.g. $\Delta_{t+1}^{I} > 0$, hence
\begin{align}\nonumber
    \underset{i\in[N]}{\min}\mathbf{m}_t[i] &\geq \left(\frac{1}{\sqrt{16(s+1)\ln^2(t)}} + \frac{1}{\sqrt{l}}\right)^{-2}\\\nonumber
    &\geq \frac{4(s+1)\ln^2(t)}{(1/2+\psi_{\min})^2}.
\end{align}
Through appropriate change of variable and application of L'Hôpital's rule, it is straightforward to verify that $\underset{\tau\rightarrow\infty}{\lim}~\ln(\varrho_{\tau})/\ln^2(\tau) = 0$ and therefore $\underset{\tau\rightarrow\infty}{\lim} H_{\tau}^{(o)} = 0~\forall o\in[N]$. The negative derivative of $H_t^{(o)}$ indicates that $G_t$ is monotonically falling as well. Thus when $\varphi_{\mathcal{H}_\theta}$ is sufficiently small, there is $T_{\mathcal{B},\mathcal{H}_{\theta},N}\in\mathbb{N}$ such that for $t > T_{\mathcal{B},\mathcal{H}_{\theta},N}$, \[4G_t < 4G_{T_{\mathcal{B},\mathcal{H}_{\theta},N}} < \Delta_{\min}.\]
To subsume the technical initialization phase of the algorithm, in which a first sample of each instantaneous reward was collected, we also demand $T_{\mathcal{B},\mathcal{H}_{\theta},N}\geq\lceil N/s\rceil$.\\\\
Overall, together with (\ref{eq:ev2opt}) and (\ref{eq:PI}), we find that
\begin{align}\label{eq:part1-done}
    \mathbb{P}[\mathcal{V}] \leq (2 q s + 1) t^{-2 (s+1)}.
\end{align}
The same upper bound can be derived for the probability of (\ref{eq:part2}) in a similar fashion.\\\\
Finally, we consider (\ref{eq:part3}).
For every $m\in[q]$, $j\in\mathcal{I}(\mathbf{x}_{t+1})$, we have
\begin{align}\nonumber
    &{\Big\lvert}\left(\hat{\boldsymbol{\phi}}^{(m)}_{t}[j] + \mathcal{W}^{(m)_{\max}}_{t}(\mathbf{x}_{t+1})[j]\right) \\\nonumber
    &\hspace{5mm}- \left(\hat{\boldsymbol{\phi}}^{(m)}_{t}[j] - \mathcal{W}^{(m)_{\min}}_{t}(\mathbf{x}_{t+1})[j]\right){\Big\rvert} \\\nonumber
    &\hspace{10mm}\leq 4\mathbf{C}_t[j] \stackrel{(a)}{=} 4\sqrt{\frac{(s+1)\ln t}{\mathbf{m}_t[j]}} \leq 4\sqrt{\frac{(s+1)\ln t}{l}}\\ \label{eq:momcon}
    &\hspace{10mm}\stackrel{(b)}{\leq} \frac{\min\{q,s\}(\Delta_{\min} - 4\Gamma^{(q)}_{\max})}{4w_{\max}(q+1)s^{\min\{q,s\}+1}},
\end{align}
where in $(a)$ and $(b)$, we substituted the value for $\mathbf{C}_{t}[j]$ and $l$, respectively.
On the same note, we find that
\begin{align}\nonumber
    E_{t+1}(\mathbf{x}_{t+1}) &= \Delta_{t+1}^{I}\sqrt{\frac{16(s+1)\ln^2(t)}{\underset{i\in\mathcal{I}(\mathbf{x}_{t+1})}{\min}\mathbf{m}_t[i]}} \\\nonumber
    &\leq \Delta_{t+1}^{I}\sqrt{\frac{16(s+1)\ln^2(T)}{l}} \\\label{eq:extbou}
    &\leq 2 \Delta_{t+1}^{I}\psi_{\min} \leq \Delta_{\min}/4 - \Gamma_{\max}^{(q)}.
\end{align}
Setting $\mathcal{Z}^{({\boldsymbol{\alpha}})}_{\max}[j] = \left(\hat{\boldsymbol{\phi}}^{({\boldsymbol{\alpha}}[j])}_t[j] + \mathcal{W}^{({\boldsymbol{\alpha}}[j])\max}_t(\mathbf{x}_{t+1})[j]\right)$ and $\mathcal{Z}^{({\boldsymbol{\alpha}})}_{\min}[j]$ accordingly, we can conclude
\begin{align}\nonumber
    &I^{\max}_{t+1}(\mathbf{x}_{t+1}) - I^{\min}_{t+1}(\mathbf{x}_{t+1}) + E_{t+1}(\mathbf{x}_{t+1}) \\ \nonumber
    &= \sum_{\lvert{\boldsymbol{\alpha}}\rvert \leq q} \frac{D^{{\boldsymbol{\alpha}}}\hat{\mathbf{r}}^{(t)}_{\mathbf{x}_{t+1}}(\mathbf{a}_{t+1})}{{\boldsymbol{\alpha}}!} \\\nonumber
    &\hspace{1mm}{\Big (} \prod\limits_{j\in\mathcal{I}(\mathbf{x}_{t+1})}\mathcal{Z}^{({\boldsymbol{\alpha}})}_{\max}[j] - \prod\limits_{j\in\mathcal{I}(\mathbf{x}_{t+1})}\mathcal{Z}^{({\boldsymbol{\alpha}})}_{\min}[j] {\Big )} + E_{t+1}(\mathbf{x}_{t+1})\\ \nonumber
    &\stackrel{(c)}{\leq} \sum_{\lvert{\boldsymbol{\alpha}}\rvert \leq q} \frac{\lvert D^{{\boldsymbol{\alpha}}}\hat{\mathbf{r}}^{(t)}_{\mathbf{x}_{t+1}}(\mathbf{a}_{t+1})\rvert}{{\boldsymbol{\alpha}}!} \left\|\mathcal{Z}^{({\boldsymbol{\alpha}})}_{\max} - \mathcal{Z}^{({\boldsymbol{\alpha}})}_{\min}\right\|_1 \\\nonumber
    &\hspace{5mm}+ E_{t+1}(\mathbf{x}_{t+1})\\ \nonumber
    &\stackrel{(\ref{eq:momcon})}{<} \sum\limits_{\substack{\lvert{\boldsymbol{\alpha}}\rvert \leq q\\\mathbf{x}_{t+1}\odot{\boldsymbol{\alpha}}={\boldsymbol{\alpha}}}} \frac{w_{\max}}{{\boldsymbol{\alpha}}!} s \frac{\min\{q,s\}(\Delta_{\min} - 4\Gamma^{(q)}_{\max})}{4 w_{\max}(q+1)s^{\min\{q,s\}+1}} \\\nonumber
    &\hspace{5mm}+ E_{t+1}(\mathbf{x}_{t+1}) \\ \nonumber
    &\stackrel{(d)}{\leq} \frac{\min\{q,s\}(\Delta_{\min} - 4\Gamma^{(q)}_{\max})}{4(q+1)s^{\min\{q,s\}}} \sum_{k=0}^{q}\frac{s^{k}}{k!} + E_{t+1}(\mathbf{x}_{t+1}) \\\nonumber
    &\leq \Delta_{\min}/4 - \Gamma^{(q)}_{\max} + E_{t+1}(\mathbf{x}_{t+1})\\ \nonumber
    &\stackrel{(\ref{eq:extbou})}{\leq} \Delta_{t+1} - \Delta_{\min}/2 - 2\Gamma^{(q)}_{\max} \\ \label{eq:part3-done}
    &= \mu(\mathbf{x}^{\ast}) - \mu(\mathbf{x}_{t+1}) - \Delta_{\min}/2 -2\Gamma^{(q)}_{\max},
\end{align}
where $(d)$ is a combinatorial result from induction over $s$ and application of the binomial theorem. Moreover, $(c)$ results from $\mathcal{Z}^{({\boldsymbol{\alpha}})}_{\max}[j], \mathcal{Z}^{({\boldsymbol{\alpha}})}_{\min}[j] \in [-1, 1]$ by decomposing the difference of products into a sum of products of differences. Hence, we conclude that (\ref{eq:part3}) never happens.\\\\
By using $(\ref{eq:event})$, (\ref{eq:part1-done}) and (\ref{eq:part3-done}), we achieve
\begin{align} \nonumber
\mathbb{E}&[\mathscr{T}_{i}(T)] 
\leq \left \lceil {\frac{4 (s+1) \ln^{2}{T}}{ \psi_{\min}^{2} }} \right \rceil \\ \nonumber
&+ \sum_{t=T_{\mathcal{B},\mathcal{H}_{\theta},N}}^{\infty} \sum_{m_{v_{1}}=1}^{t} \dots \sum_{m_{v_{s}}=1}^{t} \sum_{m_{u_{1}}=l}^{t} \dots \sum_{m_{u_{s}}=l}^{t} \\\nonumber
&\hspace{5mm}2 (2 q s + 1) t^{-2 (s+1)} \\ \nonumber
&\leq {\frac{4 (s+1) \ln^{2}{T}}{ \psi_{\min}^{2} }} + 1 + (2 q s + 1) \sum_{t=1}^{\infty} 2t^{-2} \\
&\leq {\frac{4 (s+1) \ln^{2}{T}}{ \psi_{\min}^{2} }} + 1 + \frac{\pi^{2}}{3} (2 q s + 1).
\end{align}
Therefore, the expected regret is upper bounded as
\begin{align} \nonumber
    \mathcal{R}(T)
    &\leq T_{\mathcal{B},\mathcal{H}_{\theta},N} \Delta_{\max} + \Delta_{\max} \sum_{i = 1}^{N} \mathbb{E} [\mathscr{T}_{i}(T)] \\ \nonumber
    &\leq T_{\mathcal{B},\mathcal{H}_{\theta},N} \Delta_{\max} \\\nonumber
    &\hspace{1mm}+ \sum_{i=1}^{N} \left[ {\frac{4 (s+1) \ln^{2}{T}}{ \psi_{\min}^{2} }} + 1 + \frac{\pi^{2}}{3} (2 q s + 1) \right] \Delta_{\max} \\\nonumber
    &\leq T_{\mathcal{B},\mathcal{H}_{\theta},N} \Delta_{\max} \\\nonumber
    &\hspace{1mm}+ \left[ {\frac{4 (s+1) \ln^{2}{T}}{ \psi_{\min}^{2} }} + 1 + \frac{\pi^{2}}{3} (2 q s + 1) \right] N \Delta_{\max}.
\end{align}
\begin{flushright}
$\blacksquare$
\end{flushright}
\subsubsection{Dependencies}
\label{rem:dep}
i) The maximum norm in regards to an orthonormal basis of $\mathcal{H}^{(i)}$ (w.r.t. the inner product on $\mathcal{H}^{(i)}$), $i\in[N]$, is bounded by $\theta$ for functions from $\mathcal{H}^{(i)}_{\theta}$ similar to Lemma \ref{lem:dish}. Maximum norms in other bases are therefore bounded on a combination of $\theta$ and $\dim(\mathcal{H}^{(i)})$. The number of nonzero coefficients in the basis representation is proportional to the highest degree among the empirical graphs corresponding to $(\hat{\mathbf{F}}_t)_{t\in[T]}$. For a polynomial Hilbert space, the coefficients of any $\hat{\mathbf{r}}^{(t)}$ are then limited by a power of this quantity determined by $q$, as are the linearly mapped partial derivatives of $\hat{\mathbf{r}}^{(t)}$ subsumed by $w_{\max}$. Instead of $q$, for any RKHS, an upper bound can be stated in terms of the maximum path length among the empirical graphs, which equals the longest function chain in the inverse $(\mathbf{I}-\hat{\mathbf{F}}_t)^{-1}$. The latter should be combined with a norm-threshold for edge recognition in the algorithm. Both upper limits on $w_{\max}$ are independent of $T$ and grow with the density of the graphs.

ii) Similar to i), the gradient bound $\eta$ introduced in (\ref{eq:Ey}) depends on $\theta$ as well as the maximum path length $p$ and maximum degree $d$ of the empirical graphs belonging to $(\hat{\mathbf{F}}_t)_{t\in[T]}$. The number of relevant indices in all summations in (\ref{eq:Ey})-(\ref{def:opt_gua}) is at most $d$. We can then find $T_{\mathcal{B},\mathcal{H}_{\theta},p,d}\in\mathbb{N}$ independent of $N$ for which (\ref{eq:part1-done}) is replaced by $\mathbb{P}[\mathcal{V}]\leq(2 q s + N) t^{-2(s+1)}$ to accommodate the summation of the $N$ overall rewards.

%
\subsection{REMAINING PROOFS}
\label{sec:AddProofs}
\subsubsection{Proof of Theorem \ref{thm:1} for SSEM-UCB-JO}
\label{proof:thm:2}
We only need to update the kernel optimization guarantees utilized in proof \ref{sec:ProofThm1}. To this end, we use the same variable terms with $l$ changed by a factor of $\sqrt{T}/\ln(T)$. We write $\hat{\mathbf{y}}_{t+1}=(\mathbf{I}-\hat{\mathbf{F}}_{t+1})^{-1}$. We understand the exogenous vectors observed up to time $t$ as sampled from the empirical distribution $\mathcal{B}_t$. Optimizing over this entire data set results in a complexity bound of
\begin{align}\nonumber
   \varrho_t &\geq \mathbb{P}(\mathbb{E}_{\mathcal{B}_t} [(\hat{\mathbf{F}}_{t}(\mathbf{y}(\mathbf{b}))[i]-\mathbf{F}(\mathbf{y}(\mathbf{b}))[i])^2] \\\nonumber
    &- \mathbb{E}_{\mathcal{B}_t} [(\mathbf{F}(\mathbf{y}(\mathbf{b}))[i]-\mathbf{F}_{\mathcal{B}_{t}}(\mathbf{y}(\mathbf{b}))[i])^2]> H_{t})
\end{align}
for every $i\in[N]$, with $\mathbf{F}_{\mathcal{B}_t}[i]=\underset{h\in\mathcal{H}^{(i)}_{\theta}}{\arg\min}~d_{\mathcal{B}_t}(h\mathbf{y}, \mathbf{F}\mathbf{y}[i])$ and \[H_{t} = M_1\frac{1}{\sqrt{t}} + M_2\sqrt{\frac{\ln (2/\varrho_t)}{2t}}\] as opposed to (\ref{eq:rad}). In combination with the definition of $\mathbf{F}_{\mathcal{B}_t}$, for every $o\in[N]$ we know that
\begin{align}\nonumber
    &\varrho_t \geq \mathbb{P}(\mathbb{E}_{\mathcal{B}_t} [(\hat{\mathbf{F}}_{t}(\mathbf{y}(\mathbf{b}))[i]-\mathbf{F}(\mathbf{y}(\mathbf{b}))[i])^2] \\\nonumber
    &\hspace{0.3mm}- \mathbb{E}_{\mathcal{B}_t} [(\mathbf{F}(\mathbf{y}(\mathbf{b}))[i]-\mathbf{F}_{\mathcal{B}_{t}^{(\neg o)}}(\mathbf{y}(\mathbf{b}))[i])^2]> H_{t}) \\\nonumber
    &\hspace{0.1mm}\geq\mathbb{P}(\mathbb{E}_{\mathcal{B}_t^{(o)}} [(\hat{\mathbf{F}}_{t}(\mathbf{y}(\mathbf{b}))[i]-\mathbf{F}(\mathbf{y}(\mathbf{b}))[i])^2] \\\nonumber
    &\hspace{0.3mm}- \mathbb{E}_{\mathcal{B}_t^{(o)}} [(\mathbf{F}(\mathbf{y}(\mathbf{b}))[i]-\mathbf{F}_{\mathcal{B}_{t}^{(\neg o)}}(\mathbf{y}(\mathbf{b}))[i])^2]> \frac{t}{\mathbf{m}_t[o]}H_{t}) \\\nonumber
    &\hspace{0.1mm}\geq\mathbb{P}(\mathbb{E}_{\mathcal{B}_t^{(o)}} [(\hat{\mathbf{F}}_{t}(\mathbf{y}(\mathbf{b}))[i]-\mathbf{F}(\mathbf{y}(\mathbf{b}))[i])^2] \\\nonumber
    &\hspace{5mm}> \frac{t}{\mathbf{m}_t[o]}H_{t}+\varphi_{\mathcal{H}_{\theta}}^2),
\end{align}
with $\mathcal{B}_t^{(\neg o)}$ formed from samples in time steps where $o$ was not selected, and altered objective gap \[\varphi_{\mathcal{H}_{\theta}} = \underset{i\in[N]}{\max}\underset{\mathbf{x},\mathbf{x}'\in\mathcal{X}}{\max} ~\sqrt{\mathbb{E}[(\mathbf{F}-\mathbf{F}_{\mathcal{B}^{\mathbf{x}'}})(\mathbf{y}_{\mathbf{x}}(\mathbf{b}))[i]^2]}.\]
Following the proof of Theorem \ref{thm:1}, we only need to establish that $\underset{\tau\rightarrow\infty}{\lim}\frac{\tau}{\mathbf{m}_\tau[o]}H_{\tau}=0$ to be able to infer a sufficient $T_{\mathcal{B},\mathcal{H}_{\theta},N}\in\mathbb{N}$ accordingly. From the algorithm's selection criterion, we see that $\underset{i\in[N]}{\min}\mathbf{m}_t[i] \geq \frac{4(s+1)\sqrt{t}\ln t}{(1/2+\psi_{\min})^2}$. Combined with $\underset{\tau\rightarrow\infty}{\lim}\sqrt{\tau\ln(\varrho_\tau)}/(\sqrt{\tau}\ln(\tau))=0$, the result follows. \hfill $\blacksquare$
\subsubsection{Proof of Proposition \ref{pro:1}}
\label{proof:pro:1}
For a cyclic SEM, convergence of $\hat{\mathbf{F}}_{t}$ towards $\mathbf{F}$ on the sample measure $\mathbb{E}_{\mathcal{B}_t}$ implies convergence in an associated weighted $L^2$-norm since the operators are analytic. As all $\mathcal{H}^{(i)}$, $i\in[N]$, are finite-dimensional, all norms on them are equivalent and the operators in particular converge with respect to their Hilbert space norm. For any RKHS, the convergence must be uniform, as
\begin{align*}
	\lvert(\hat{\mathbf{F}}_t-\mathbf{F})(\mathbf{y})[i]\rvert &= \lvert\langle\hat{\mathbf{F}}_t[i]-\mathbf{F}[i],K_{\mathbf{y}}^{(i)}\rangle_{\mathcal{H}^{(i)}}\rvert \\
	&\leq \|K_{\mathbf{y}}^{(i)}\|_{\mathcal{H}^{(i)}}\|\hat{\mathbf{F}}_t[i]-\mathbf{F}[i]\|_{\mathcal{H}^{(i)}}
\end{align*}
by Cauchy-Schwarz, where $K_{\mathbf{y}}^{(i)}=K^{(i)}(\mathbf{y},\cdot)$ is the kernel function of $\mathcal{H}^{(i)}$ and $\mathbf{y}$ is a solution to the SEM of $\mathbf{F}$ for any $\mathbf{b}\in[0,1]^N$. The inverse function theorem then implies convergence of the inverses $(\mathbf{I}-\hat{\mathbf{F}}_t)^{-1}$ at the same time-rate, although with different constant factors affecting only $T_{\mathcal{H}_{\theta},\mathcal{B},N}$. Past the operator accuracy, the remainder of proof \ref{sec:ProofThm1} does not exploit acyclicity and can be replicated here. \hfill $\blacksquare$

\subsubsection{Proof of Corollary \ref{thm:3}}
\label{proof:thm:3}
We have
\begin{align}\nonumber
    \mathcal{R}(T) &= \sum_{t=1}^{T} \Delta(\mathbf{x}_t) \\ \nonumber
    &= \sum_{\substack{t\in[T]\\4\Gamma^{(q)}_{\max} < \Delta(\mathbf{x}_t)}} \Delta(\mathbf{x}_t) + \sum_{\substack{t\in[T]\\4\Gamma^{(q)}_{\max} \geq \Delta(\mathbf{x}_t)}} \Delta(\mathbf{x}_t) \\ \nonumber
    &\hspace{-6mm}\stackrel{(\text{Thm. \ref{thm:1}})}{\leq} \left[ {\frac{4 (s+1) \ln^{2}{T}}{ {\psi_{\min}^{(q,s)}}^{2} }} + 1 + \frac{\pi^{2}}{3} (2 q s + 1) \right] N \Delta_{\max}^{(q)} \\ \nonumber
    &\hspace{7mm} + T_{\mathcal{B},\mathcal{H}_{\theta},N} \Delta_{\max}^{(q)} + 4 \Gamma^{(q)}_{\max} T. 
\end{align}
\hfill $\blacksquare$
\subsubsection{Proof of Corollary \ref{cor:1}}
\label{proof:cor:1}
All central moments $\boldsymbol{\phi}^{(p)}[j] = \mathbb{E}_{\mathcal{B}_j}[(b-\mathbb{E}_{\mathcal{B}_j}[c])^p]$ of normal distributions can be calculated from their respective variance:
\[\forall p\in\mathbb{N}, j\in[N]: \boldsymbol{\phi}^{(2p)}[j] = (2p-1)!!{\boldsymbol{\phi}^{(2)}}[j]^{p}\] \[\text{ and } \boldsymbol{\phi}^{(2p-1)}[j] = 0.\]
Replicating the proof of Theorem \ref{thm:1}, only the unbiased estimates of mean and variance need to be taken into account in (\ref{eq:VarEst}). Further, (\ref{eq:momcon}) needs to be extended by the argument that for $\vert a\rvert,\lvert b\rvert,\lvert c\rvert \leq 0.5$, $p\in\mathbb{N}$, by induction we have \[\lvert(b-a)^p - (c-a)^p\rvert \leq p\lvert(b-a)-(c-a)\rvert=p\lvert b-c\rvert\] and that all $D^{{\boldsymbol{\alpha}}}\hat{\mathbf{r}}^{(t)}_{\mathbf{x}_{t-1}}$ for $\lvert{\boldsymbol{\alpha}}\rvert\leq q$, $t\in[T]$, are $s w_{\max}$-Lipschitz continuous. \hfill $\blacksquare$
\subsubsection{Proof of Corollary \ref{cor:2}}
\label{proof:cor:2}
The $\dim(\mathcal{H}^{(i)})$ samples obtained while playing column vector $\mathbf{M}[i]$ are with probability $1$ in general position and allow for exact identification of $\mathbf{F}[i]$ for any $i\in[N]$. Thus, the terms inherited from the optimization uncertainty after the initialization phase can be omitted from the bound established in Theorem \ref{thm:1}. \hfill$\blacksquare$
\subsubsection{Proof of Theorem \ref{thm:4}}
\label{proof:thm:4}
We set $\tilde{\boldsymbol{\phi}}^{(m)} = \mathbb{E}_{\mathbf{b}\sim\mathcal{B},\boldsymbol{\epsilon}\sim\mathcal{N}}[(\mathbf{b}+\boldsymbol{\epsilon}-\mathbf{a}_{t+1})^{m}]$ for the noisy moments about $\mathbf{a}_{t+1}$ of the exogenous signals. Only definition (\ref{eq:J_t}) (and its subsequent appearances) in the proof of Theorem \ref{thm:1} has to be altered to feature the noisy moments instead of the pure signal moments. We note that the expected model noise can be incorporated into the endogenous operator since $\mathbf{F}+\mathbb{E}[\mathcal{M}]\in\mathcal{H}_{\theta}$. The expectation in the objective gap shall be taken with respect to the noise as well.\hfill$\blacksquare$
%
\subsection{EXPERIMENTAL ADDENDUM}
\label{sec:expadd}
\subsubsection{Synthetic}
\label{subsec:ads}
As a measure of nonlinearity, we only accepted problem instances in which the true reward function's first degree Taylor approximation's remainder empirically averaged above $10\%$ of the maximal reward in absolute value. Concerning IGP-UCB, the underlying RKHS is defined by a squared exponential kernel $k$ and the hyperparameters $B$, $\lambda$, $R$, $\delta$ (in the notation of \textup{\cite{pmlr-v70-chowdhury17a}}) were set to $(\mathbb{E}[\mathbf{r}_{\mathbf{x}}(\mathbf{b})])_{\mathbf{x}\in\mathcal{X}}(k(\mathbf{x},\mathbf{x^{\prime}}))_{\mathbf{x},\mathbf{x^{\prime}}\in\mathcal{X}}(\mathbb{E}[\mathbf{r}_{\mathbf{x}}(\mathbf{b})])_{\mathbf{x}\in\mathcal{X}}^{\top}$, $\underset{\mathbf{x}\in\mathcal{X}}{\max}(\underset{\mathbf{b}\in[0,1]^N}{\max}\mathbf{r}_{\mathbf{x}}(\mathbf{b}) - \mathbb{E}_{\mathbf{b}\sim\mathcal{B}}[\mathbf{r}_{\mathbf{x}}(\mathbf{b})])$, $\sqrt{\lambda}$, $0.1$, respectively.

The Taylor centers of SSEM-UCB-JO were placed at the means of the instantaneous rewards. The RKHSs were comprised of the elemental component functions allowed during the DAG generation. Their norm was surrogated with the norm-equivalent $\ell_1$-norm. For runtime considerations, we implemented the differentiation numerically, infrequently retrained the adjacency graph only if it inadequately described new sampling data (similar to the change point detection in \textup{\cite{piecewiseSEM}} and \textup{\cite{delayedSEM}}), and set minuscule thresholds for the predicted component function coefficients and derivative values. For each super arm, an optimistic reward projection can be calculated more efficiently by considering each addend of the Taylor approximation separately, while the theoretical regret bound remains applicable. This way, each confidence score can be set to its absolutely lower or higher bound depending on the sign of the accompanying partial derivative value and whether any of the confidence intervals surround $0$. We also record that scaling down $\mathbf{C}_t$ or decreasing the external UCB's weight over time often led to faster convergence within the given time horizon.

In a second trial instance, we alter the setup described in Section \ref{subsec:synt} by replacing the scaled periodic $x\mapsto\sin(2\pi x)$ components with quadratic and hyperbolic tangent functions. We also include SSEM-UCB-Norm (Algorithm \ref{Alg:SSEM-UCB-Norm} in Appendix \ref{sec:spa}) in our analysis. Unlike SSEM-UCB-JO, it is better suited to automatic differentiation in its implementation but only entails two confidence scores per super arm to identify by optimization.

A third session additionally sees the instantaneous rewards drawn from arcsine distributions with means in range $[1, 1.5]$ and standard deviations in $[0.2, 0.6]$. The number of base arms is raised to $15$ while the edge density is decreased to $0.15$ and function coefficients are generated from $[0.05, 0.15]$.

Comparison of the average results over $3$ runs displayed in Figure \ref{fig:synth1} and Figure \ref{fig:synth2} with those shown in Section \ref{subsec:synt} indicates that SSEM-UCB-JO favors decentralized signal distributions with increased variances that intensify the nonlinearity of the system, while SSEM-UCB-Norm specifically addresses normally distributed signals. A dominant external exploration of valid options accommodates structural uncertainty but can incur undue initial regrets. The balance can be tuned towards earlier internal exploration according to the causal simplicity of the problem. Another trade-off lies in computation time, as low degrees of SSEM and SGB are much less demanding than higher ordered SSEM and METS. The expense is amplified by greater connectivity in the underlying graph, albeit pressing the qualitative advantage over algorithms that falsely rely on reward separability or linearity. SSEM displays robustness outside the requirements set for the theoretical study, as the hyperbolic tangent kernel is not positive-definite.
\begin{figure}[!t]
    \centering
    \includegraphics*[width=0.95\columnwidth]{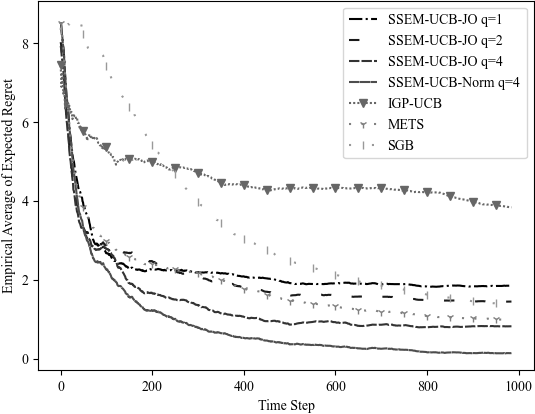}
    \caption{Policies' Mean Regret on Synthetic Normal Data.}
    \label{fig:synth1}
\end{figure}
%
\begin{figure}[!t]
    \centering
    \includegraphics*[width=0.94\columnwidth]{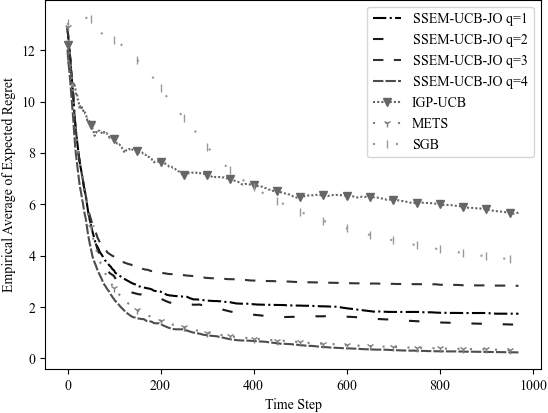}
    \caption{Policies' Mean Regret on Synthetic Arcsine Data.}
    \label{fig:synth2}
\end{figure}
%
%
\subsubsection{Real World}
\label{subsec:adr}
\begin{table}[!ht]
\caption{Station Acronyms}
\begin{center}
\begin{tabular}{|l|c|}
\hline
Abbreviation & Train Stop \\
\hline
\hline
EISN & Eisenach \\
\hline
GOTH & Gotha \\
\hline
LEIN & Leinefelde \\
\hline
NEUD & Neudietendorf \\
\hline
NORD & Nordhausen \\
\hline
ERFU & Erfurt Hbf \\
\hline
BEBR & Bebra \\
\hline
GÖTT & Göttingen \\
\hline
WEIM & Weimar \\
\hline
BADH & Bad Hersfeld \\
\hline
NORT & Northeim (Han) \\
\hline
GOSL & Goslar \\
\hline
KASH & Kassel Hbf \\
\hline
JENW & Jena West \\
\hline
JENP & Jena Paradies \\
\hline
SAAL & Saalfeld (Saale) \\
\hline
KASW & Kassel-Wilhelmshöhe \\
\hline
JENG & Jena-Göschwitz \\
\hline
HALB & Halberstadt \\
\hline
FULD & Fulda \\
\hline
\end{tabular}
\end{center}
\end{table}
%
\begin{figure}[thb]
    \centering
    \includegraphics*[width=0.95\columnwidth]{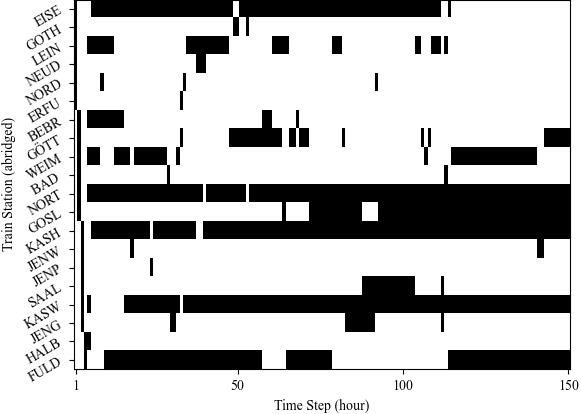}
    \caption{Super Arm Selection over Time.}
    \label{fig:real}
\end{figure}
%
For broader exploration, we chose the alternative external UCB $E_t(\mathbf{x}) = \sum\limits_{i\in[N]}\mathbf{x}[i]4\mathbf{C}_t[i]\sqrt{\ln(t)}/s$ (see Section \ref{subsec:ext}). As in the synthetic case, we directly computed the partial derivative values around the means numerically and tuned the regularization parameter adaptively. To further increase the computational speed of the SSEM-UCB-JO implementation, we approximated the inverse $(\mathbf{I}-\hat{\mathbf{F}}_{t})^{-1}$ at each time step $t \leq 342$ with a residual neural network $\mathbf{N}_{\Theta}$ fit to the loss obtained from the $L^2$-norm of the structural equations $\|\mathbf{N}_{\Theta}(\mathbf{z})-\hat{\mathbf{F}}_{t}(\mathbf{N}_{\Theta}(\mathbf{z})) - \mathbf{z}\|_2$ on random samples $\mathbf{z}\in\mathbb{R}^N$. Prior normalization of the signals can stabilize the training process against outliers. Scaling the average reward signal also regulates the frequency of the exploration-exploitation trade-off. Optimal coefficients on the circular graph can be found efficiently through, e.g., ADMM.
The super arm exploration of SSEM-UCB-JO during the first $150$ time steps is visualized in Figure \ref{fig:real}.

Although the perpetual graph learning deployed by SSEM-UCB-JO does make it agnostic to daily changes in the unshuffled data, in the future, incorporating dedicated dynamic methods \textup{\cite{delayedSEM, piecewiseSEM, topology}} would be of particular interest to trace relations on different time frames. Larger-scale experiments could draw from further sources of periodically disclosed German rail network statistics.
%

\end{document}